\documentclass[a4paper,12pt,reqno,nocut]{article}
\usepackage[utf8]{inputenc}
\usepackage[T1]{fontenc}
\usepackage{amsmath,amssymb,amsthm,mathtools}
\usepackage{authblk}
\usepackage{mathrsfs,bm,color}
\usepackage{fancyhdr,lastpage}
\usepackage[shortlabels]{enumitem}
\usepackage{calc}
\usepackage{bbm}
\usepackage{graphicx}
\usepackage{upgreek}
\usepackage{lmodern}
\usepackage[normalem]{ulem}
\usepackage{verbatim}
\usepackage{bbm}
\usepackage{stmaryrd}
\usepackage{amsmath}
\usepackage{amssymb}
\usepackage{dsfont}
\usepackage{amsthm}
\usepackage{hyperref}
\hypersetup{
	colorlinks,
	linkcolor={red!50!black},
	citecolor={red!80!black},
	urlcolor={blue!80!black}
}

\usepackage{xcolor}

\usepackage{graphicx}

\usepackage{amsfonts}%
\usepackage{dsfont}

\allowdisplaybreaks
\usepackage[hmargin=2cm,vmargin=2cm]{geometry}

\numberwithin{equation}{section}

\usepackage{thmtools}
\declaretheorem[thmbox=L,name=Theorem,numberwithin=section]{theo}
\declaretheorem[name=Proposition,thmbox=M,numberlike=theo]{proposition}
\declaretheorem[name=Lemma,numberlike=theo]{lemma}
\declaretheorem[name=Definition,numberlike=theo]{defn}
\declaretheorem[name=Corollary,numberlike=theo]{corollary}

\theoremstyle{plain}

\makeatletter
\renewenvironment{proof}[1][\proofname]{\par
	\pushQED{\qed}%
	\normalfont \topsep6\p@\@plus6\p@\relax
	\trivlist
	\item[\hskip\labelsep
	\sffamily\bfseries
	#1\@addpunct{.}]\ignorespaces
}{%
	\popQED\endtrivlist\@endpefalse
}
\makeatother

\newcommand{\jump}{{\vskip 0.3cm \noindent }}

\renewcommand{\and}{\mbox{ and }}

\newcommand{\R}{\mathbb{R}}

\def\as{{ \mathrm{a.s.}  }}

\newcommand{\vect}[1]{\ensuremath{\boldsymbol{\mathbf{#1}}}}
\newcommand{\rdmvect}[1]{\ensuremath{\bm{#1}}}
\newcommand{\mat}[1]{\ensuremath{\boldsymbol{\mathbf{#1}}}}
\newcommand{\rdmmat}[1]{\ensuremath{\bm{#1}}}

\renewcommand{\det}{\ensuremath{\mathrm{det}}}

\renewcommand{\leq}{\leqslant}
\renewcommand{\geq}{\geqslant}
\renewcommand{\epsilon}{\varepsilon}

\def\R{{\mathbb R}}

\newcommand{\matV}{\mat{V}}
\newcommand{\matOme}{\mat{\Omega}_K}
\newcommand{\matGam}{\mat{\Gamma}_K}

\newcommand{\vgz}{\vect{g}(z)}

\begin{document}

\title{Spectral Phase Transition and Optimal PCA in Block-Structured Spiked models}

\author{Pierre Mergny\thanks{IdePHICS laboratory, \'Ecole F\'ed\'erale Polytechnique de Lausanne, Switzerland. Email: \texttt{pierre.mergny@epfl.ch}}, 
	Justin Ko\thanks{ Department of Statistics and Actuarial Science, University of Waterloo, Canada.},
	Florent Krzakala\thanks{IdePHICS laboratory, \'Ecole F\'ed\'erale Polytechnique de Lausanne, Switzerland.}
 }
\date{}

\maketitle
\begin{abstract}
We discuss the inhomogeneous spiked Wigner model, a theoretical framework recently introduced to study structured noise in various learning scenarios, through the prism of random matrix theory, with a specific focus on its spectral properties. Our primary objective is to find an optimal spectral method and to extend the celebrated \cite{BBP} (BBP) phase transition criterion ---well-known in the homogeneous case--- to our inhomogeneous, block-structured, Wigner model.  We provide a thorough rigorous analysis of a transformed matrix and show that the transition for the appearance of 1) an outlier outside the bulk of the limiting spectral distribution and 2) a positive overlap between the associated eigenvector and the signal, occurs precisely at the optimal threshold, making the proposed spectral method optimal within the class of iterative methods for the inhomogeneous Wigner problem.
\end{abstract}


\section{Introduction}
\label{sec:introduction}
The statistical challenge of inferring a low-dimensional signal from a noisy, high-dimensional observation is ubiquitous across statistics, probability, and machine learning.  Spiked random matrix models have recently gained extensive interest, serving as a valuable platform for exploring this issue \cite{donoho1995adapting,peche2014deformed,lesieur2017constrained}. A prominent example is the spiked Wigner model, where a rank one matrix is observed through a component-wise homogeneous noise, that has been studied extensively in random matrix theory \cite{BBP}.

Most models, with the \emph{spiked Wigner model} at the forefront, have focused however on scenarios where the noise is “homogeneous", aiming to understand how the performance of the inference depends on the noise level. 
Yet in practice, datasets are inherently structured and the exploration of inhomogeneity plays a pivotal role in unraveling their complexities. A prototypical model to study this phenomenon is to improve the aforementioned spiked Wigner model by introducing a block structure in the noise, a model which has been recently introduced in a series of papers \cite{behne2022fundamental,alberici2021multi,alberici2022statistical,AJFL_inhomo} and that arises in many different learning contexts such as community detection \cite{behne2022fundamental,AJFL_inhomo}, deep Boltzmann machines \cite{alberici2021multi_boltzmann}, or the dense limit of the celebrated degree-corrected stochastic block model \cite{AJFL_inhomo,karrer2011stochastic}.

Our goal in this paper is to apply rigorous random matrix theory to such “inhomogenous" spiked models, and to provide an optimal reconstruction method from a spectral algorithm, to generalize the seminal work of \cite{BBP} (BBP) to inhomogenous matrices.

\paragraph{Settings and open questions ---} The model is defined in practice by multiplying the Wigner matrix by a variance-profile matrix, namely one would like to infer in the high-dimensional setting $N \gg 1$, the underlying signal $\vect{x} \in \mathbb{R}^{N}$ from the noisy matrix observation:
\begin{equation} 
\label{eq:low-rank_inhomogenous}
		\rdmmat{Y} = \sqrt{\frac{1}{N}} \vect{x}\vect{x}^{\top} + \rdmmat{H} \odot \left(\mat{\Delta}^{\odot 1/2} \right),
\end{equation}
where $\mat{\Delta}$ is a block-constant matrix, $\odot$ denotes the Hadamard product $(\mat{A} \odot \mat{B})_{ij} = A_{ij} B_{ij} $, $\mat{A}^{\odot   \alpha }$ is the Hadamard power ($(\mat{A}^{\odot \alpha})_{ij} = A_{ij}^{\alpha} $) and $\rdmmat{H}$ is a real-valued symmetric matrix with independent Gaussian entries of unit variance above the diagonal. The precise assumptions on $\vect{x}$ and $\mat{\Delta}$ are postponed to the next section. 

\medskip

The study of the inhomogeneous model of Eq.~\eqref{eq:low-rank_inhomogenous} from a Bayes-optimal point of view has been performed in a series of work in the asymptotic limit $N \to \infty$  in \cite{AJFL_inhomo,behne2022fundamental,MourratXia-tensor, MourratXiaChen-tensor} who characterized the fundamental information-theoretic limit of reconstruction in this model. Recently \cite{pak2023optimal} discussed algorithmic performances of the Approximate Message Passage (AMP in short) algorithm introduced in particular a simpler variant of this algorithm for this problem (linearized AMP, see also \cite{aubin2019spiked,mondelli2018fundamental,maillard2022construction}) and \emph{conjectured} it to be optimal for detection. The linearized version of AMP can be shown to be equivalent to performing principal component analysis (PCA) on a linear transformation of the matrix $\rdmmat{Y}$, namely estimate the signal $\vect{x}$ using  the top eigenvector of the matrix: 
\begin{align}
\label{eq:def_init_mat_tilde_Y}
  \rdmmat{\Tilde{Y}}  &:= \rdmmat{Y} \odot   \left(\mat{\Delta}^{\odot -1} \right) - \frac{1}{\sqrt{N}} \mathrm{Diag} \left(\mat{\Delta}^{\odot -1} \, \vect{1} \right) \, ,
\end{align}
where $\vect{1}:= (1,\dots,1) \in \mathbb{R}^N$ and for $\vect{v} \in \mathbb{R}^N$,  $\mathrm{Diag}(\vect{v})_{ij}:= v_i \mathbb{I}_{i=j}$ is the corresponding diagonal matrix made of the entries of its argument. 

Due to the block structure of the noise in 
 the random matrix $\rdmmat{\Tilde{Y}}$, tracking the (possible) outliers in its spectrum is a challenging problem and in  \cite{pak2023optimal} the associated phase transition for the existence of an outlier has only been conjectured to occur when a specific parameter (described below) reaches the value one.  While it was hinted by \cite{pak2023optimal} that the spectral method associated with the matrix \eqref{eq:def_init_mat_tilde_Y} may be an efficient one for inhomogenous problems, it was left an open problem.

\paragraph{Our contributions ---} In this work, we step up to this challenge, and provide a rigorous analysis of the spectral method for the inhomogeneous spiked Wigner models of Eq.~\eqref{eq:def_init_mat_tilde_Y}. More specifically, 
\begin{enumerate}
    \item We proved the conjecture of \cite{pak2023optimal} and show that the spectral method for the matrix  $\rdmmat{\Tilde{Y}}/\sqrt{N}$ defined by (\ref{eq:def_init_mat_tilde_Y}), has a phase transition at the predicted  {\rm algorithmic threshold} \cite{AJFL_inhomo}.
    \item This makes PCA applied to (\ref{eq:def_init_mat_tilde_Y}) optimal in detection in the inhomogenous spiked model (at least in the absence of a statistical to computation gap, see \cite{AJFL_inhomo}). Our results could also be used as spectral start \cite{mondelli2021approximate} for AMP algorithms \cite{pak2023optimal}, making them optimal in terms of MMSE for these problems.
    \item We obtained a complete characterization of the associated overlap with the hidden signal, and their phase transition, further generalizing the BBP results from homogeneous matrices to inhomogeneous ones.
\end{enumerate}
In particular, we prove that this phase transition occurs when the top eigenvalue of a  certain symmetric matrix  $\matOme$, containing all information of the model (proportion of each community, the noise level inside and between different communities), takes the value one. This critical value $\lambda_1(\matOme)=1$ corresponds to the phase transition of AMP for this model \cite{pak2023optimal}, and more generically to the algorithmic threshold of efficient inference of the problem \cite{AJFL_inhomo} (see e.g.  \cite{zdeborova2016statistical,bandeira2018notes,bandeira2022franz} for discussion of the algorithmic to statistic gap in such high-dimensional problems). As a consequence, our result shows that  PCA on the transformed matrix $\rdmmat{\Tilde{Y}}$ also achieves optimal efficient detection for the model \eqref{eq:low-rank_inhomogenous}.

To the authors' best knowledge,  our work is the first to provide an exact asymptotic of the phase transition of heterogeneous spiked Wigner models, both for the top eigenvalue and the overlap. To obtain these results, our study relies on a detailed analysis of large block-structure random matrix models that go beyond the standard tools used in the spiked homogeneous models and that can be of independent interest for analyzing structured matrices in machine learning.

\paragraph{Other related works ---}
The task of factorizing low-rank matrices to find a hidden signal in data, ranging from sparse PCA to community detection and sub-matrix localization. Many variants of the homogeneous problem have been studied in the high-dimensional limit, see e.g. \cite{DBLP:journals/corr/DeshpandeAM15,lesieur2017constrained,barbier2018rank,lesieur2017constrained,10.1214/19-AOS1826,lelargemiolanematrixestimation,barbier2020information}. The inhomogeneous version was introduced and studied
in \cite {behne2022fundamental,alberici2021multi,alberici2022statistical,AJFL_inhomo,pak2023optimal}.

Spectral methods are a popular tool for solving rank-factorization problems \cite{donoho1995adapting,peche2014deformed,BBP}, and this has triggered a large amount of work in the random matrix theory (RMT) community. Spiked models with homogeneous noise are well understood and have been studied extensively in RMT, it is well known that as one decreases the noise level (that is the variance of the entries of the Wigner matrix), there is a  phase transition in the behavior of the top eigenvalue which detaches from the semi-circle distribution, see for example \cite{Peche_largest_2005,feral_largest_2007,capitaine_largest_2009,Capitaine12fluctuations,pizzo_finite_2013,renfrew_finite_2_2013} and also \cite{baik_eigenvalues_2006,paul_asymptotics_2007,bloemendal_limits_2013,benaych-georges_eigenvalues_2011,benaych-georges_singular_2012,guionnet_spectral_2023} for other variants of these models.

The properties of the spectrum of variance-profile Wigner matrices (without any small-rank perturbation) are also well understood \cite{ajanki_singularities_2017,ajanki_universality_2017,ajanki_quadratic_2019,alt_dyson_2020} even though there is, in general, no closed-form solution to the limiting spectral distribution unlike the semi-circle distribution for Wigner matrices. To the best knowledge of the authors, only the two works \cite{bigot_freeness_2020,lee_phase_2023} performed a spectral analysis of spiked block-structure Wigner models but in some specific regimes. \cite{bigot_freeness_2020} studied the spectrum of a (generalization of) models of the form given directly by Eq.~\eqref{eq:low-rank_inhomogenous} (instead of the transform $\rdmmat{\Tilde{Y}}$ of Eq.~\eqref{eq:def_init_mat_tilde_Y} of this paper) using free probability tools but without studying the overlap or having a closed expression for the outlier, as it is done in this paper. \cite{lee_phase_2023} looked at the behavior of the top eigenvalue in a (restricted version)  of $(2 \times 2)$ blocks and our result provides the extension to the $(K \times K)$ case with $K$ fixed, together with an expression for the overlap. Eventually, we mention the series of work  \cite{huang_algorithmic_2023,huang_strong_2023} on multi-species spherical models involving random matrix tools that are closely related to the ones of this paper. 

Spectral methods are also important as providing a {\it warm start} for other algorithmic approaches. This is the case, for instance, for approximate message passing algorithms (AMP) \cite{mondelli2021approximate}. AMP has attracted a lot of attention in the high-dimensional statistics and machine learning community, see e.g.  \cite{donoho2009message,BayatiMontanari,rangan2011generalized,rangan2012iterative}, and was written, for the inhomogeneous spiked problem, in \cite{pak2023optimal}. AMP algorithms and the corresponding weak recovery conditions were later developed for the much larger class of matrix tensor product models \cite{rossetti2023approximate}, which includes the inhomogeneous spiked problem among many others. AMP algorithms are optimal among first-order methods \cite{celentano2020estimation}. Equipped with the warm start provided by our theorem,  their reconstruction threshold thus provides a bound on the algorithmic performances. The link between  AMP and spectral methods was discussed, for instance, in \cite{saade2014spectral,lesieur2017constrained,aubin2019spiked,mondelli2018fundamental,mondelli2022optimal,maillard2022construction,venkataramanan2022estimation}.

\section{Main Theoretical Results}
\subsection{Assumptions and Notations}

We assume that $\vect{x} \sim \pi^{\otimes N}$ where $\pi$ is a distribution with mean zero and variance one.

We denote by $\mathbb{H}_{\pm} := \{ z \in \mathbb{C} , \pm \mathfrak{Im} z >0 \}$ and by $\overline{\mathbb{H}_\pm}:=\mathbb{H}_{\pm} \cup \mathbb{R} $.

For two vectors of equal dimension, we introduce the partial order relation $\vect{x} \succ \vect{y}$ (resp. $\vect{x} \succeq \vect{y}$) to denote the usual relations $x_i >y_i$ (resp. $x_i \geq y_i$) for all $i$.  We say that a vector $\vect{x}$ is positive (resp. non-negative) if $\vect{x} \succ \vect{0} $ (resp. $\vect{x} \succeq \vect{0} $). We will repeatedly use the simple identity that if $\max_i x_i \leq c$ then $\vect{x} \preceq c \vect{1} $.   We denote by $| \vect{x} | = (|x_i|)_i$ the component-wise absolute value of vector. We denote by $\mat{D}_{\vect{x}} = \mathrm{Diag}( (x_i) )$ the diagonal matrix obtained from the vector $\vect{x}$.

In the following, when considering a matrix $\mat{A}$  of size $(N \times N)$ (or a vector of in $\mathbb{R}^N$) we omit the dependency in the dimension  $N$ and statement involving such a matrix $\mat{A}$ with $N \to \infty$ implicitly refer to a sequence $\mat{A} \equiv \mat{A}_{(N)}$ of such matrices.  $(K \times K)$ matrices and $K$-dimensional vectors with $K$ fixed are usually (but not always when it is clear from the context) denoted with a $K$ index (e.g. $\mat{A}_K$ and $\vect{a}_K$) to differentiate them with $(N \times N)$ matrices and $N$-dimensional vectors. In particular, $\mat{I}_K$ is the identity matrix of size $(K\times K)$, $\vect{0}_K = (0,\dots,0)$  and $\vect{1}_K = (1,\dots,1)$ are the $K$-dimensional vectors of zeroes and ones respectively.

For a symmetric matrix $\mat{A}$, we denote by $(\lambda_i(\mat{A}))_{1 \leq i \leq N}$ its eigenvalues in decreasing order, in particular $\lambda_1(\mat{A})$ is the largest eigenvalue. We also denote by $\mu_{\mat{A}} := \frac{1}{N} \sum_{i=1}^N \delta_{\lambda_i(\mat{A})}$  the \emph{empirical spectral distribution}. We say that an eigenvalue $\lambda_i(\mat{A})$ of $\mat{A}$ separates from the bulk if  $\mu_{\mat{A}}$ converges weakly almost surely to a \emph{limiting spectral distribution} $\mu_A$ and $\lim_{N \to \infty} \mathrm{min}_{x \in \mathrm{Supp}(\mu_A)} \mathrm{dist}(x,\lambda_i(\mat{A}) ) >0$.

We say that a partition of $[N] := \{1, \dots, N\}$ is divided into $K$ groups if there exist $B_i \neq \emptyset$ such that  $B_1 \cup \dots \cup B_K = [N]$. Since the eigenproblem is invariant by permutations, we assume without loss of generality that for all $k < l$, we have $i <j$ if $i \in B_k$ and $j \in B_l$, and that inside each block $B_k$, the entries are in increasing order. Additionally, we denote by $\rho_k(N) = |B_k|/N$ the proportion of each group (thus $\sum_{k=1}^K\rho_k(N)  =1$) and we assume that as $N \to \infty$ with $K$ fixed,  the limit are well defined: $\rho_k := \lim_{N \to \infty} \rho_k(N) \in (0,1)$. 

In order to ease notation, we denote by $\mat{\Sigma} \equiv \mat{\Delta}^{\odot -1}$. We also denote by $\mat{S}_K$ the $(K \times K)$ symmetric matrix with non-negative entries $ s_{kl} := (\mat{S}_K)_{kl}$ and assume that  the $(N \times N)$ matrix $\mat{\Sigma}$ is block-constant with entries inside each block given by the $s_{ij}$, that is there exists a partition $B_1 \cup \dots \cup B_K = [N]$ such that $ \mat{\Sigma}_{ij} = s_{kl}$ for all $i \in B_k, j \in B_l$. 

For simplicity, we will assume $s_{kl} >0$ for any $k,l$ to avoid degenerate cases, although we believe that this will not change our final result.

In terms of the matrix $\mat{\Sigma}$, we can write the output matrix $\rdmmat{\Tilde{Y}}$ as:
\begin{align}
\label{eq:YeqXplusZ}
    \rdmmat{\Tilde{Y}} &= \rdmmat{X}  + \mat{Z} \, , 
\end{align}
with
\begin{align}
\label{eq:def_matX}
  \rdmmat{X}    
  &= 
  \rdmmat{H} \odot   \left(\mat{\Sigma}^{\odot 1/2} \right) - \frac{1}{\sqrt{N}} \mathrm{Diag} \left(\mat{\Sigma} \vect{1} \right) \, ,
\end{align}
and
\begin{align}
\label{eq:def_matZ}
      \mat{Z} = \left(   \sqrt{\frac{1}{N}} \vect{x}\vect{x}^{\top} \right) \odot \mat{\Sigma} \, .
\end{align}

We define the ($K \times K)$ symmetric matrix encoding all parameters ($\rho_k$, $s_{kl}$) of the model as:
\begin{align}
\label{eq:def_matOme}
    \mat{\Omega}_K := \mat{D}_{(\vect{\rho}^{\odot 1/2})} \, \mat{S}_K \mat{D}_{(\vect{\rho}^{\odot 1/2})}
\end{align}
and by
\begin{align}
\label{eq:def_matGam}
   \mat{\Gamma}_K :=\mat{S}_K \mat{D}_{\vect{\rho}}  =  \mat{D}_{(\vect{\rho}^{\odot -1/2})} \,  \mat{\Omega}_K \mat{D}_{(\vect{\rho}^{\odot 1/2})} \, . 
\end{align}
The matrices $\mat{\Gamma}_K$ and $\mat{\Omega}_K$ are similar and thus share the same (real) eigenvalues. Note that all elements of $\mat{\Omega}_K$ are positive and by Perron-Frobenius Theorem (see e.g. Chap.~8 of \cite{meyer_matrix_2023}) the top eigenvalue $\lambda_1(\mat{\Omega}_K)$ is simple and corresponds to the operator norm. 

Eventually, we introduce the vector $\vect{\mu} \in \R^K$ of the overlap between the top eigenvector $\vect{u}_1$ of $\rdmmat{\Tilde{Y}}/\sqrt{N}$ and the vector in each community as: 
\begin{align}
\label{eq:def_overlap}
    \vect{\mu} := \left(   \Big \langle  \frac{\vect{x}_1}{ \| \vect{x}_1 \|} , \vect{u}_1 \Big \rangle  , \dots,  \Big  \langle   \frac{\vect{x}_K}{ \| \vect{x}_K \|} , \vect{u}_1  \Big \rangle   \right), 
\end{align}
where $(\vect{x}_k)_i := x_i$ for $i \in B_k$ and $(\vect{x}_k)_i = 0$ for $i \notin B_k$.
Note that the overlap vector $\vect{\mu}$ contains by definition the information on the recovery of the signal for each community $( x_i)_{i \in B_k}$. From it, one can obtain the overlap with the entire vector, $q:= \langle \vect{u}_1, \vect{x}/\| \vect{x} \| \rangle$, as we have $q = \langle \vect{\mu}, \vect{\rho}^{\odot 1/2} \rangle$.   

Eventually, we mention by universality for variance-profile Wigner matrices \cite{ajanki_universality_2017} that our result can be easily extended to non-Gaussian noise. Similarly, a brief examination of the proof indicates that one can also relax the assumption on $\vect{x}$, for example assuming only that $\|(x_i)_{i \in B_k} \|/| B_k| \xrightarrow[N \to \infty ]{\mathrm{a.s}} 1$ sufficiently fast.

\subsection{Phase Transition for Top Eigenvalue and Top Eigenvector}
\jump 
Our main theorem (Thm.~\ref{thm:main_thm} below) of this work indicates that the top eigenvalue $\lambda_1(\matOme)$ plays the role of the signal-to-noise ratio (SNR) for this block-structure spiked model where the threshold value $\lambda_1(\matOme) =1$ separates a phase where the top eigenvector has zero overlap with the signal vector to a phase where it has a positive overlap with the signal. The behavior of the SNR with respect to the parameters $(\rho_k , s_{kl})$ of the model are postponed in App.~\ref{sec:prop_SNR}.

\begin{theo} We have the following phase transition for the outlier and the overlap 
\label{thm:main_thm}
\begin{itemize}[itemsep=1pt,leftmargin=1em]
    \item  For $\lambda_1(\mat{\Omega}_K) \leq 1$: 
    \begin{enumerate}[itemsep=1pt,leftmargin=1em]
        \item  there is asymptotically no eigenvalue separating outside of the bulk;
        \item   $\vect{\mu}  \xrightarrow[N \to \infty]{\as} \vect{0}_K = (0,\dots,0)$.
    \end{enumerate}  
    \item  For $\lambda_1(\mat{\Omega}_K) > 1$:
    \begin{enumerate}[itemsep=1pt,leftmargin=1em]
        \item $\lambda_1(\rdmmat{\Tilde{Y}}/\sqrt{N}) \xrightarrow[N \to \infty]{\as} 1$ and separates from the bulk;
        \item $|\vect{\mu} | \xrightarrow[N \to \infty]{\as}  C^{-1/2} \, \mat{D}_{(\vect{\rho}^{\odot 1/2})} (\vect{1}_K - \vect{g}(1))$
        with the positive constant $C$ given by
        \begin{align}
        \label{eq:def_C}
            C := \langle \vect{1}_K - \vect{g}(1), \matGam^{\top} \mat{D}_{(\vect{\rho} \odot \vect{y})}  \matGam ( \vect{1}_K - \vect{g}(1) ) \rangle \, ,
        \end{align}
        where $ \vect{y} = (\matGam - \mat{D}_{\vect{g}(1)}^{-2})^{-1} \vect{1}_K $ and $\vect{g}(1)$ is the  continuation at $z=1$ of the solution of the quadratic vector equation of Prop.~\ref{prop:QVE}.
    \end{enumerate}
\end{itemize}
\end{theo}
We highlight that the expression for the overlap after the transition involves a quantity $\vect{g}(1)$ which is a root of a quadratic vector equation, and in general the solutions of this system of equations do not admit a closed-form expression. 

In Fig.~\ref{fig:top_eigenvalue}, we illustrated the behavior of the theoretical limiting spectral distribution and top eigenvalue of the matrix $\rdmmat{\tilde{Y}}/\sqrt{N}$ before, at, and after the critical value $\lambda_1(\mat{\Omega}_K) = 1$ for specific values of the parameters. In Fig.~\ref{fig:overlap}, we plotted the value of the overlaps for different models and different values of the $\lambda_1(\mat{\Omega}_K)$.

We recall that inference on inhomogeneous models of the form given by Eq.~\eqref{eq:low-rank_inhomogenous} are known to exhibit a \emph{statistical-to-computational gap}, see Lem.~2.6 of  \cite{AJFL_inhomo}  and \cite{aubin2019spiked,bandeira2018notes} for other models exhibiting such phenomenon. Our main result proves Conjecture 2.6 in \cite{pak2023optimal} and shows that the threshold for PCA on the transformed matrix $\rdmmat{\tilde{Y}}$ matches the optimal algorithmic recovery threshold for these inhomogeneous models, see Part 2 of Lem.~2.6 in \cite{AJFL_inhomo}. We also note from Prop. 2.18 of \cite{AJFL_inhomo} that the thresholds for spectral methods on the original matrix $\rdmmat{Y}$ or on $\rdmmat{Y} \odot (\mat{\Delta}^{\odot -1/2})$ (which \emph{a priori} removes the inhomogeneity in $\rdmmat{Y}$ at the cost of transforming the signal vector $\vect{x}$) are always strictly worse than this optimal algorithmic recovery. This, in particular, demonstrates the relevance of using $\rdmmat{\Tilde{Y}}$, and the interest in the guaranty we provide, for spectral methods, rather than those two matrices.

\begin{figure*}[ht]
\vskip 0.2in
\begin{center}
\includegraphics[width=0.32\linewidth]{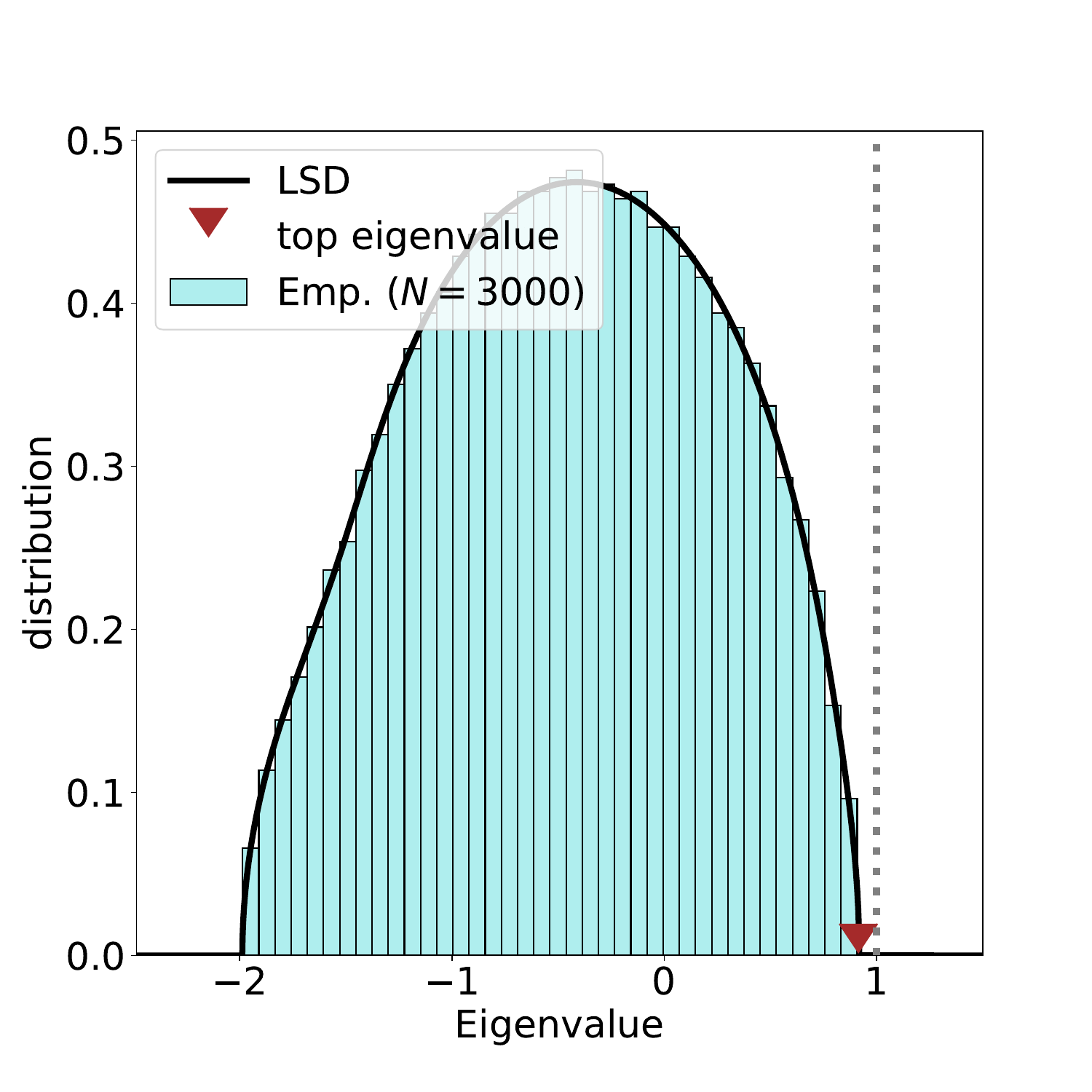}
\includegraphics[width=0.32\linewidth]{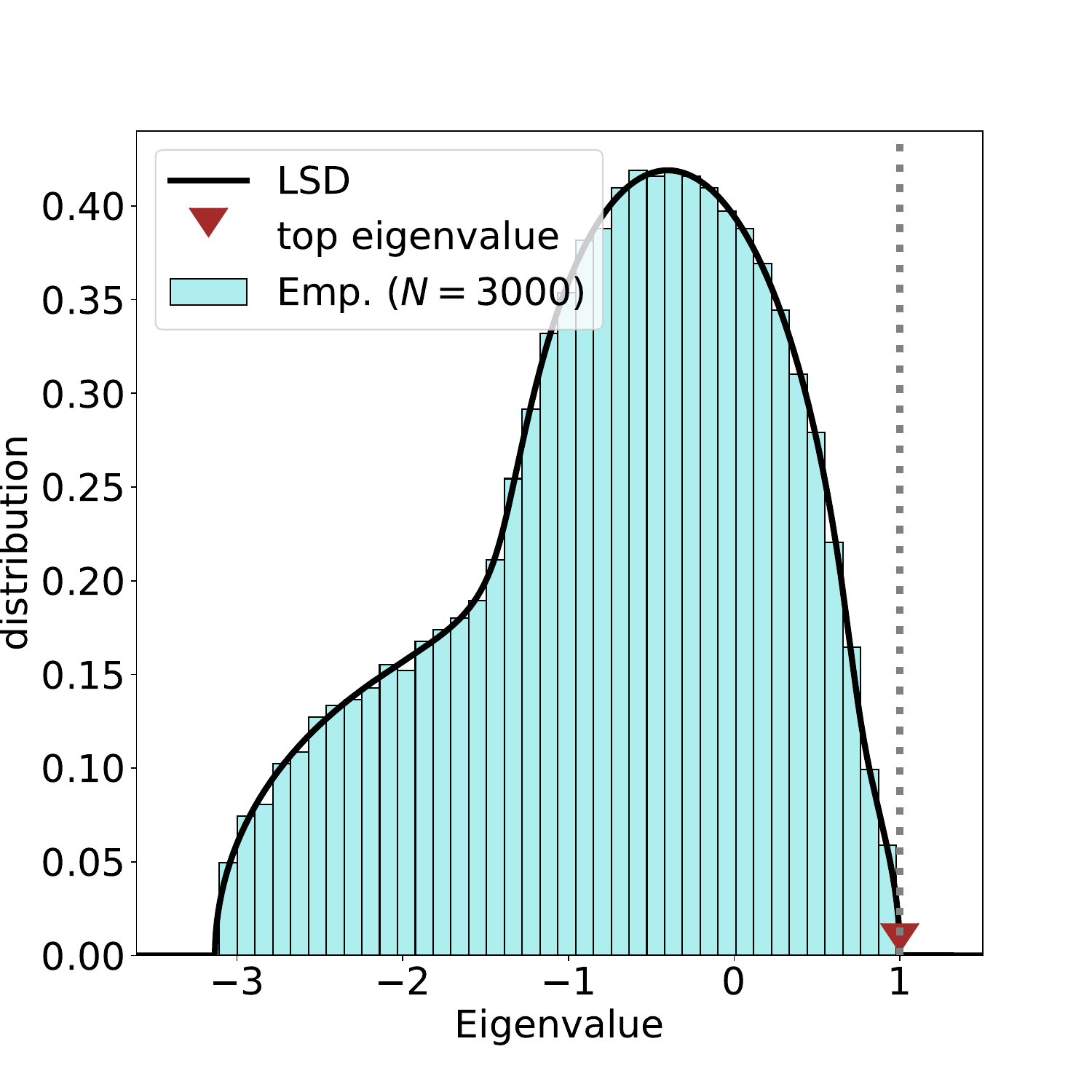}
\includegraphics[width=0.32\linewidth]{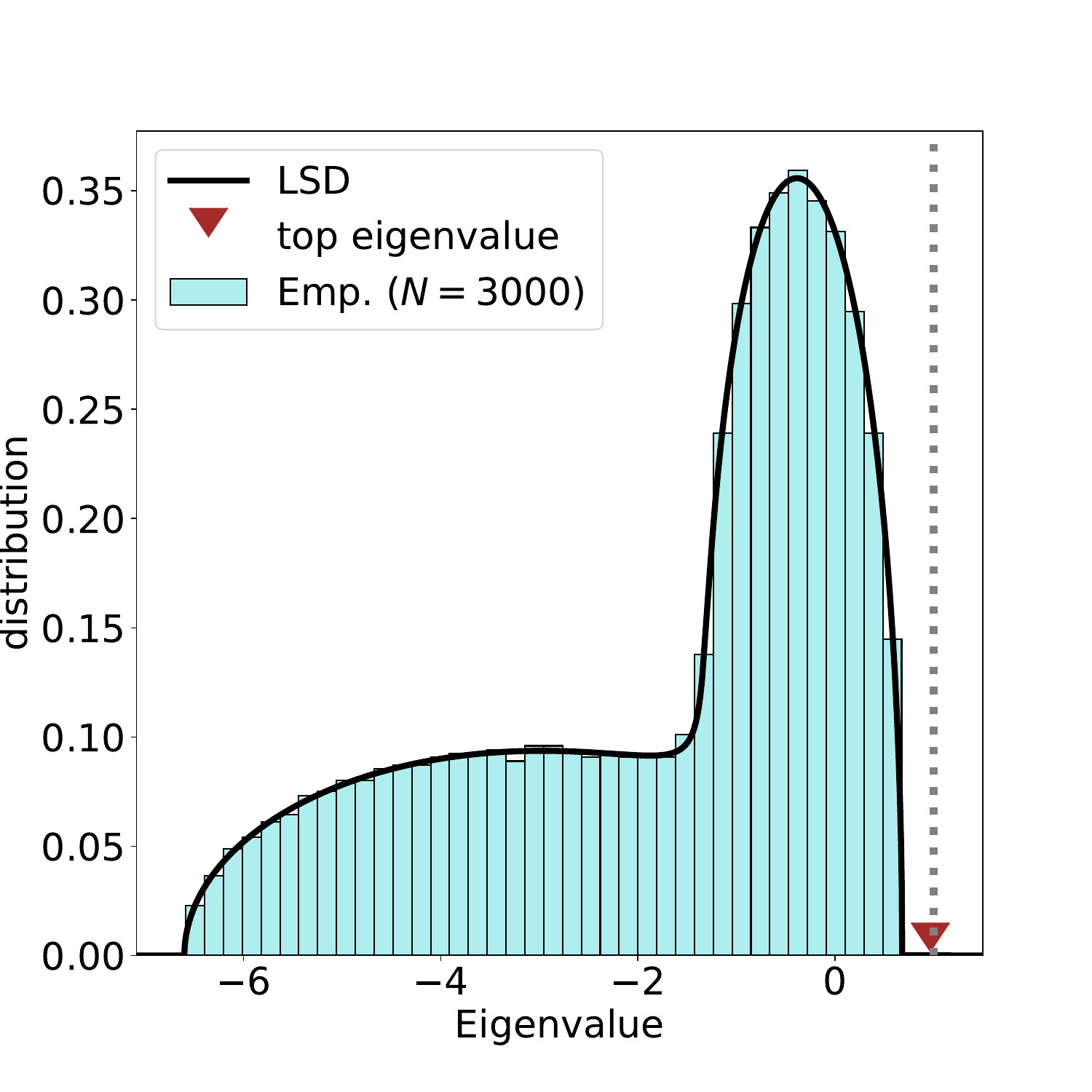}
\caption{\textbf{Eigenvalue Distribution and Top Eigenvalue Position for Different Value of $\lambda_1(\matOme)$.} The three figures correspond to a model with $K=2$, $\vect{\rho} = (1/2,1/2)$ and   $ \mat{S}_K = {\small\begin{pmatrix}
    t & 1/2 \\
    1/2 & 1/4
\end{pmatrix}}$ with a value of the parameter $t$ different in each figure and chosen such that \textbf{(Left)} $\lambda_1(\matOme) = 0.5 <1$, \textbf{(Center)} $\lambda_1(\matOme) = 1.0$ and  \textbf{(Right)} $\lambda_1(\matOme) = 3.0 >1$. The black curve corresponds to the theoretical value of the limiting spectral distribution (LSD) obtained by solving the QVE of Eq.~\eqref{eq:QVE} numerically, while the colored histogram corresponds to the empirical distribution of a sample $\rdmmat{\Tilde{Y}}/\sqrt{N}$ with $N=3000$,  and the red triangle corresponds to the empirical value of its largest eigenvalue. The signal $\vect{x}$ has been sampled from a standard normal distribution. Before the transition (left), the rightmost edge is below one and there is no outlier; at the transition (center), the rightmost edge touches the value one; and after the transition (right) there is an outlier at one.   }
\label{fig:top_eigenvalue}
\end{center}
\vskip -0.2in
\end{figure*}

\begin{figure*}[ht]
\vskip 0.4in
\begin{center}
\includegraphics[width=0.49\linewidth]{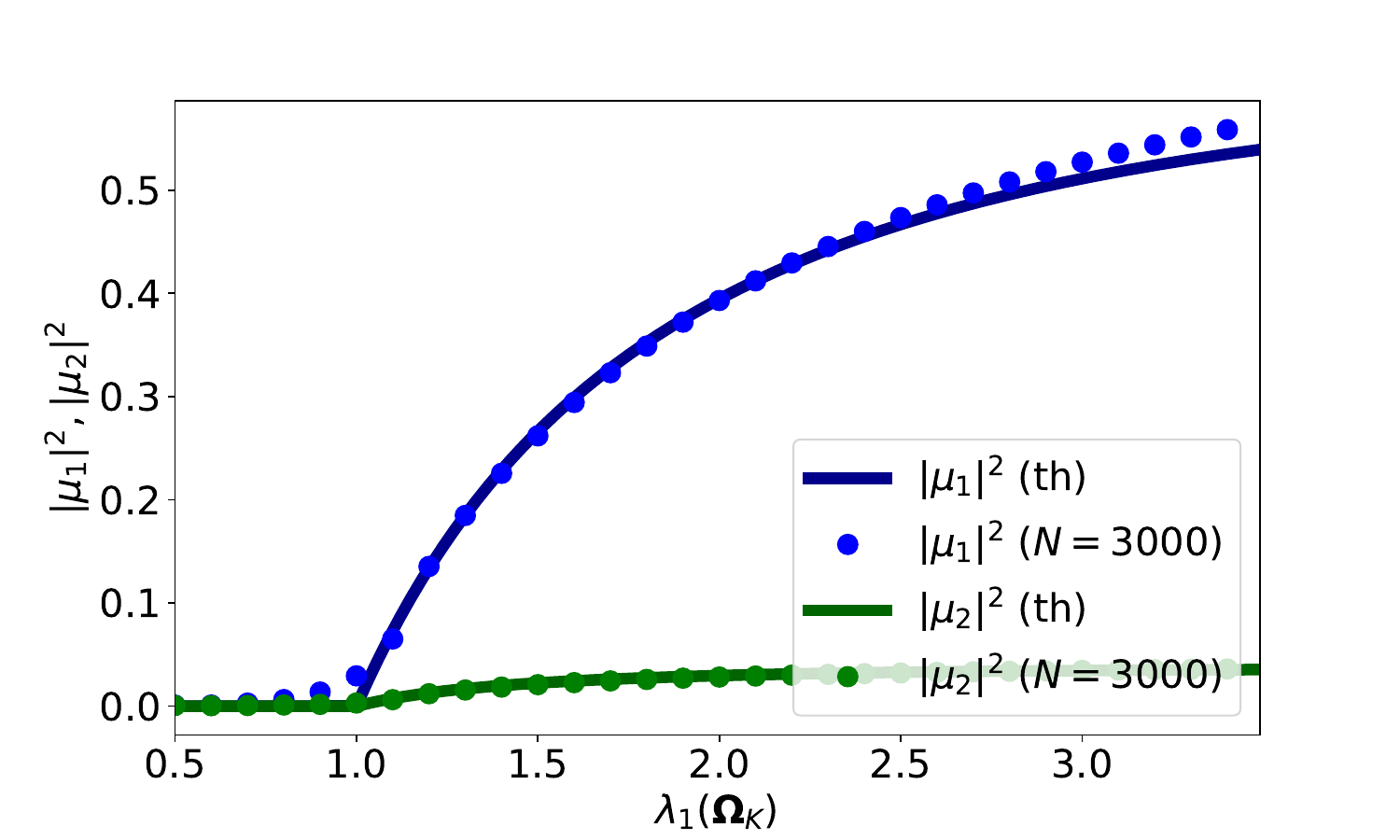}
\includegraphics[width=0.49\linewidth]{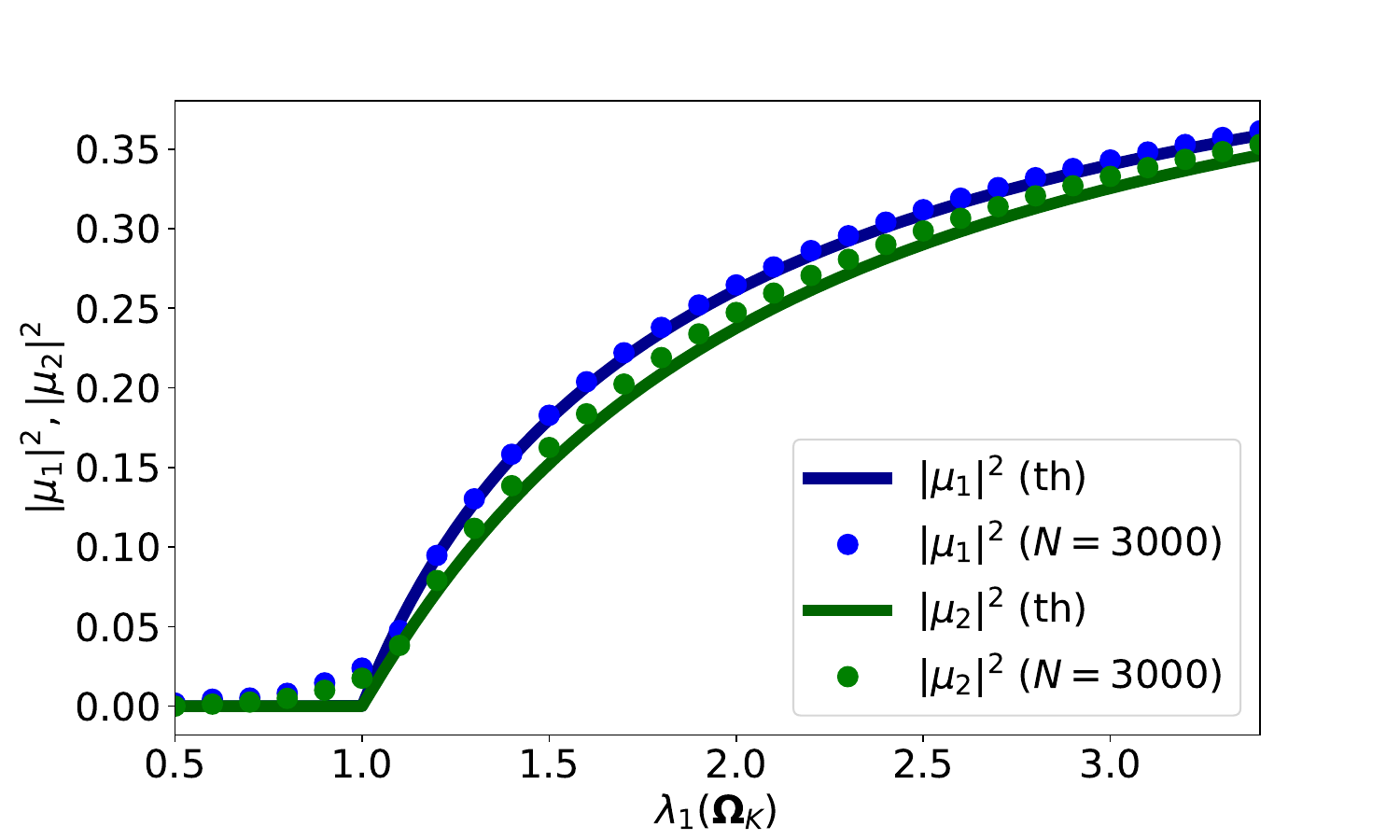}
\caption{\textbf{Overlap Vector for Different Value of $\lambda_1(\matOme)$ and Different Models.} \textbf{(Left)} Value of the square of the overlaps for a model with $K=2$, $\vect{\rho} = ( 1/2, 1/2)$ and $\mat{S}_K = {\small\begin{pmatrix}
    t & 1/2 \\
    1/2 & 1/2
\end{pmatrix}}$ where the range of the parameter $t$ is set such that $\lambda_1(\matOme)$ varies from $0.5$ to $3.5$. \textbf{(Right)} Value of the square of the overlaps for a model with $K=2$, $\vect{\rho} = (1/2, 1/2)$ and $\mat{S}_K = {\small\begin{pmatrix}
    1 & t \\
    t & 1/2
\end{pmatrix}}$ where the range of the parameter $t$ is set such that $\lambda_1(\matOme)$ varies from $0.5$ to $3.5$. In both cases, the dots represent an average over $10$ samples of the empirical value of the overlaps for $N=3000$, with $\vect{x}$ sampled from a standard Gaussian distribution. }
\label{fig:overlap}
\end{center}
\vskip -0.2in
\end{figure*}

\section{Outline of the Proof}

We break down the proof of  Thm.~\ref{thm:main_thm} into several steps by first establishing the spectral properties of the matrices $\rdmmat{X}/\sqrt{N}$ defined in Eq.~\eqref{eq:def_matX} and $\mat{Z}$ defined in Eq.~\eqref{eq:def_matZ}, separately. Combining these properties, we prove that if there is an outlier in $\rdmmat{\Tilde{Y}}/\sqrt{N}$, then a certain $(K \times K)$ matrix must have an eigenvalue at one.  For $\lambda_1(\matOme) <1$ we show that this condition is not achievable, while for  $\lambda_1(\matOme) >1$, we show that this condition is satisfied and sufficient by constructing explicitly the corresponding eigenvector. We also describe how one can obtain the overlap vector $\vect{\mu}$.

\subsection{Properties of the Limiting Spectral Distribution and its Stieltjes Transform}
\paragraph{Quadratic Vector Equation --- }

The matrix $\rdmmat{\Tilde{Y}}/\sqrt{N}$ is a fixed-rank perturbation of the  variance-profile Wigner matrix $\rdmmat{X}/\sqrt{N}$. By Weyl's interlacing theorem, as $N$ goes to infinity its limiting spectral distribution (LSD) is thus equal to  $\mu_X$, the LSD of $\rdmmat{X}/\sqrt{N}$. The latter is well-understood and has been studied extensively in the series of works \cite{ajanki_singularities_2017,ajanki_universality_2017,ajanki_quadratic_2019}, which tackle a more general setting than the block structure of this paper. Even in the block-structure setting, there is, in general, no closed-form expression for the density $\mu_X$.  It is however known that this distribution is supported on a finite union of compact intervals, and the behavior of this distribution at the endpoints of those intervals is also well understood. 

In the following, we will mainly need the following characterization of the Stieltjes transform of the LSD. We recall from random matrix theory that the Stietljes transform of a continuous measure $\mu$ is defined for any $z \in \mathbb{C} \setminus \mathrm{Supp}(\mu)$ as $g(z):= \int_{\mathrm{Supp}(\mu)} (z-\lambda)^{-1} \mathrm{d} \mu(\lambda)$ and this transform uniquely characterizes the distribution $\mu$. For variance-profile Wigner matrices, this Stietljes transform is characterized by the following result. 

\begin{proposition}[Quadratic Vector Equation for the Stieltjes]
\label{prop:QVE}
Let $g_X(z)$ be the Stieltjes transform of the limiting spectral distribution $\mu_X$ of the matrix $\rdmmat{X}/\sqrt{N}$, then for any $z \in \mathbb{H}_-$, we have $g_X(z):= \sum_{k=1}^K \rho_k g_k(z)$ where $\vect{g}(z):= (g_1(z), \dots, g_K(z))$ is the unique solution of the quadratic vector equation (QVE):
\begin{equation} \label{eq:QVE}
    \vect{1}_K 
    = z \vect{g}  - \vect{g} \odot \mat{\Gamma}_K (\vect{g} - \vect{1}_K) \, , 
\end{equation}
such that  
\begin{equation}
\label{eq:condition_QVE}
    \vgz \in (\mathbb{H}_+)^K \, . 
\end{equation}
\end{proposition}
\begin{proof}
    We refer to \cite{ajanki_singularities_2017,ajanki_universality_2017,ajanki_quadratic_2019} for the complete proof and to App.~\ref{sec:proof_QVE} for a short sketch of the proof.
\end{proof}
Furthermore, from Cor.~1.10 of \cite{ajanki_universality_2017} (see also Thm. A.9 of \cite{husson_large_2022}), the top and bottom eigenvalue of $\rdmmat{X}/\sqrt{N}$ converge respectively to the rightmost and leftmost edge of $\mu_X$.

\paragraph{Evaluation on the Real Line ---}

Thm.~\ref{prop:QVE} completely characterizes the Stieltjes transform \emph{outside} the real line.
Yet, in many cases,  one is interested in understanding the behavior of this Stietljes transform precisely on this set. The property of the continuation on $\mathbb{R}$ is given the following result, from the same authors. 

\begin{proposition}[Continuation on the real line]
\label{prop:analytical_continuation_g}
    The solution $\vect{g}(.)$ of the quadratic vectorial equation of Eq.~\eqref{eq:QVE} on $\mathbb{H}_-$  extends to $1/3$-Hölder continuous function $\vect{g}:\overline{\mathbb{H}_-} \to (\overline{\mathbb{H}_+})^K$ and is analytical on $\mathbb{C} \setminus \overline{\mathrm{Supp}(\mu_X)}$. 
\end{proposition}
\begin{proof}
    See Corr.~2.7 of \cite{ajanki_singularities_2017} and Thm.~A.2 of \cite{huang_strong_2023}.
\end{proof}
Note that while Prop.~\ref{prop:analytical_continuation_g} gives a regularity property for the continuation on the real line, it does not immediately give us a way to know if a certain solution of Eq.~\eqref{eq:QVE} evaluated on $\mathbb{R} \setminus \overline{\mathrm{Supp}(\mu_X)}$ corresponds to the proper analytical continuation of $\vect{g}(.)$. Take for example $z=1$, then one can immediately check that $\vect{1}_K$ is always \emph{one} solution (among the $2^K$ possible solutions)  of Eq.~\eqref{eq:QVE}, yet does it correspond to $\vect{g}(1)$ or to a spurious solution of Eq.~\eqref{eq:QVE}? The following result shows how to discriminate between the two cases.

\begin{lemma} 
\label{prop:condition_good_sol_QVE_real_line}
Let $\lambda \in \mathbb{R} \setminus \overline{\mathrm{Supp}(\mu_X)}$ and consider any solution $\Tilde{\vect{g}}(\lambda) \in (\mathbb{R}_*)^K$ of Eq.~\eqref{eq:QVE} evaluated at $z=\lambda$, then   $\Tilde{\vect{g}}(\lambda)=\vect{g}(\lambda)$ if and only if the linear system of equations
\begin{align}
\label{eq:condition_good_solution_QVE}
   ( \mat{D}_{\Tilde{\vect{g}}(\lambda)}^{-2} -\matGam) \vect{y} =  \vect{1}_K \, , 
\end{align}
admits a unique positive solution $\vect{y} \succ \vect{0}_K$. 
\end{lemma}
\begin{proof}
    The proof of this result is given in Appendix.~\ref{sec:proof_Linear_System_QVE_real_line} and relies on studying the behavior of Eq.~\eqref{eq:QVE} as we approach the value $z=\lambda$  from the lower complex plane. See also \cite{huang_strong_2023} for a similar statement and proof. 
\end{proof}
\jump Next, we show a necessary and sufficient condition for Eq.~\eqref{eq:condition_good_solution_QVE} to have a positive solution.
\begin{proposition}
    \label{prop:eigenvalue_condition_good_sol_QVE_real_line}
The linear system of Eq.~\eqref{eq:condition_good_solution_QVE} admits a unique positive solution if and only if $\lambda_1( \mat{D}_{\Tilde{\vect{g}}(\lambda)}^2 \matOme) < 1$.     
\end{proposition}
\begin{proof}
 The proof of (a generalized version of) this result is given in App.~\ref{sec:proof_positivity_of_solution_linear_system} and relies on properties of so-called $M$- and $Z$- matrices and the use of Farka's lemma to characterize positive solutions of \eqref{eq:condition_good_solution_QVE}. 
\end{proof}
Note that in particular, the use of this lemma for $z=1$ leads to the following result, which will be useful later:
\begin{corollary}
\label{cor:behavior_g1}
    We have
    \begin{itemize}[itemsep=1pt,leftmargin=2em]
        \item For $\lambda_1(\matOme) \leq 1$: $\vect{g}(1) = \vect{1}_K$;
        \item  For $\lambda_1(\matOme) > 1$: $\vect{g}(1) \neq \vect{1}_K$.
    \end{itemize}
\end{corollary}
\jump As we will see later on, for the regime  $\lambda_1(\matOme) > 1$, we can slightly improve this result (see Lem.~\ref{lem:g1_SNR_g_1} below).

\paragraph{Behavior of the Rightmost Edge ---}
To understand if the top eigenvalue separates from the bulk, we need to understand the behavior of the rightmost edge, which is given by the following result. 
\begin{proposition}[The rightmost edge is bounded by one]
\label{prop:bound_rightmost_edge}
    Let $\lambda_r$ be the rightmost edge of the limiting spectral distribution $\mu_X$ of $\rdmmat{\Tilde{Y}}/\sqrt{N}$ then $\lambda_r \leq 1$ with equality if and only if $\lambda_1(\matOme) =1$. 
\end{proposition}
\begin{proof}
    This result is given in App.~\ref{sec:proof_bound_rightmost_edge}. It follows from a characterization of the edge of the spectrum in terms of the singularities of the solution to the quadrative vector equations as in Thm.~2.6 of \cite{ajanki_singularities_2017}
\end{proof}
Eventually, we will use the following lemma, characterizing the Stieltjes transform above the rightmost edge.
\begin{lemma}[Positivity and Monotonicity of $\vect{g}(.)$ above the rightmost edge]
\label{lem:pos_and_decreasing_sol_QVE}
    Let $\lambda_r$ be the rightmost edge of $\mu_X$, then for any $\lambda \in (\lambda_r, \infty)$ one has $\vect{g}(\lambda) \succ \vect{0}_K$ and this function is entrywise analytically decreasing on this interval.
\end{lemma}
\begin{proof}
    This is a standard result in complex analysis and follows from the fact (see \cite{ajanki_quadratic_2019}) that for any $k\in [K]$, $g_k(z):= \int (z-\lambda)^{-1} \mathrm{d}\mu_k$ for some positive measure $\mu_k$. Since $g_X = \sum_k \rho_k g_k$, we must have for any $k \in [K]$, the rightmost edge  $\lambda_{r,k}$ of $\mu_k$ satisfies the bound $\lambda_{r,k} \leq \lambda_r$. Since $g_{k}'(z) = - \int (z-\lambda)^{-2} \mathrm{d} \mu_K < 0$ and $g_k(z) \underset{z \to \infty}{\sim} 1/z$, we have the desired result. 
\end{proof}

\subsection{Eigendecomposition of the Small-Rank Matrix}
Next, we turn our attention to the small-rank matrix $\mat{Z}$ given by Eq.~\eqref{eq:def_matZ}. It is the Hadamard product of a rank-one matrix and an a priori rank-$K$ matrix and thus is a priori also of rank-$K$ by standard properties of Hadamard product (see e.g. \cite{johnson_matrix_1990}). Our next result expresses  $\mat{Z}$ as a rotation of the matrix $\mat{\Omega}_K$, up to a small error. 
 
\begin{proposition}
\label{prop:decomposition_Z}
For $N$ large enough, one has almost surely 
\begin{align}
    \mat{Z}/\sqrt{N} 
    &=  
\matV \left( \mat{\Omega}_K+ \rdmmat{E}_K \right)
{\matV}^{\top}
\, , &&
\end{align}
where $\rdmmat{E}_K$ is a $(K \times K)$ matrix such that $\| \rdmmat{E}_K \|_{\mathrm{op}} \xrightarrow[N \to \infty]{\as} 0 $, $\matV := \left[ \vect{v}^{(1)} , \dots, \vect{v}^{(K)} \right] \in \mathbb{R}^{N \times K}$ and the $\vect{v}^{(k)}$'s are orthonormal vectors satisfying the block structure property $(\vect{v}^{(k)})_i := x_i / \| ( x_i)_{i \in B_k} \|$ for $i \in B_k$ and $(\vect{v}^{(k)})_i = 0$ for   $i \notin B_k$.
\end{proposition}
\begin{proof}
    The proof of this result is postponed to  App.~\ref{sec:proof_Eigendecomposition_matZ} and follows from the block structure of $\mat{\Sigma}$.
\end{proof}
We emphasize that the condition $N$ large enough is only needed to ensure $(x_i)_{i \in B_k} \neq \vect{0}$.

\subsection{Condition for the Existence of an Outlier}
We are now ready to combine the previous results describing the spectrum and eigenvectors of the matrices $\rdmmat{X}/\sqrt{N}$ and $\mat{Z}$ to describe the spectrum of the matrix $\rdmmat{\Tilde{Y}}/\sqrt{N}$ and in particular, its potential outliers. Our first result in this direction is given by the following proposition.
\begin{proposition}[Equation for outliers]
\label{prop:eq_outlier}
If $\lambda$ is an outlier separating from the bulk, then $\lambda$ is an almost-sure solution of 
\begin{align}
\label{eq:pos_outliers}
    \det \left( \mat{I}_K - \mat{D}_{\vect{g}(\lambda)} \, \mat{\Omega}_K  \right) 
    &= 0 
    \, , &&
\end{align}
where $\vect{g}(\lambda) = (g_1(\lambda), \dots, g_K(\lambda))$, is the analytical continuation on the real line of the solution of the QVE of Prop.~\ref{prop:QVE}. 
\end{proposition}

\begin{proof}
    The proof of this result is given in App.~\ref{sec:proof_Equation_Outliers}. It relies on using (a generalized version of) the matrix determinant lemma, together with the eigendecomposition of Prop.~\ref{prop:decomposition_Z} and deterministic equivalent for the resolvent of the matrix $\rdmmat{X}/\sqrt{N}$. 
\end{proof}

\subsection{Proof of  the Non-Existence of an Outlier in the Regime \texorpdfstring{$\lambda_1(\matOme) \leq 1$}{}}

Now that we have the condition for the existence of an outlier,  we will prove that this condition cannot be satisfied for $\lambda_1(\matOme) \leq 1$.
\jump
Indeed from Prop.~\ref{prop:eq_outlier}, for $\lambda$ to be an outlier, we must have that there exists a $k \in [K]$ such that $\lambda_k( \mat{D}_{\vect{g}(\lambda)} \, \mat{\Omega}_K) = 1$ and thus, in particular, we must have
\begin{align}
\label{eq:condition_outlier_topeigenvalue}
\lambda_1(\mat{D}_{\vect{g}(\lambda)} \, \mat{\Omega}_K) &\geq 1 \, ,
\end{align}
since eigenvalues are put in decreasing order and the eigenvalues of $ \mat{D}_{\vect{g}(\lambda)} \, \mat{\Omega}_K$ are real for $\lambda \in \mathbb{R} \setminus \mathrm{Supp}(\mu_X)$. 
\jump
Now let us assume $\lambda_1(\matOme) <1$ and that there is an outlier. On the one hand, Eq.~\eqref{eq:condition_outlier_topeigenvalue} implies by operator norm inequality the bound 
\begin{align}
\label{eq:condition_outlier_abs_g1}
    | \vect{g}(\lambda) | \succ \vect{1}_K \, .
\end{align}
On the other hand, by Prop.~\ref{prop:eigenvalue_condition_good_sol_QVE_real_line} we must have $\lambda_1( \mat{D}_{\vect{g}(\lambda)}^2 \matOme) = \lambda_1( \mat{D}_{|\vect{g}(\lambda)|} \, \mat{D}_{|\vect{g}(\lambda)|}\matOme) \leq 1$ which is incompatible with the previous conditions of Eq.~\eqref{eq:condition_outlier_topeigenvalue} and Eq.~\eqref{eq:condition_outlier_abs_g1}, hence there is no outlier in this regime.
\jump
The case $\lambda_1(\matOme) =1$ is obtained by noticing that in this case, the only possible outlier is at $\lambda =1$ (see next section) but by Prop.~\ref{prop:bound_rightmost_edge} this also corresponds to the value of the rightmost edge and thus there is also no outlier in this regime.

\subsection{Proof of the Existence of a (Top) Outlier at One in The Regime \texorpdfstring{$\lambda_1(\matOme) >1$}{}}
We first prove the existence of an outlier at one and then show that this necessarily corresponds to the limiting value of the top eigenvalue of $\rdmmat{\Tilde{Y}}/\sqrt{N}$. 

\jump Let us first notice the following behavior of  $\vect{g}(1)$. 
\begin{lemma}
\label{lem:g1_SNR_g_1}
     For $\lambda_1(\matOme) >1$, we have $\vect{0}_K \prec \vect{g}(1) \prec \vect{1}_K$.
\end{lemma}
\begin{proof}
    The first inequality comes from Lem.~\ref{lem:pos_and_decreasing_sol_QVE} and the fact $\lambda_r < 1$ from Prop.~\ref{prop:bound_rightmost_edge}. The second inequality comes from the fact that since $1 \notin \mathrm{Supp}(\mu_X)$, we have from Prop.~\ref{prop:eigenvalue_condition_good_sol_QVE_real_line}  the condition $\lambda_1( \mat{D}_{\vect{g}(1)^2} \matOme) \leq 1$, together  with the assumption $\lambda_1(\matOme) >1$, this leads, by operator norm inequality, to the desired result. 
\end{proof}
\jump Next, let's introduce the vector 
    \begin{align}
    \label{eq:def_right_eigvect}
    \vect{v}^{(r)}_1 :=  \mat{D}_{(\vect{\rho}^{\odot 1/2})} (\vect{1}_K  - \vect{g}(1)). 
    \end{align}
Note that from the previous lemma, we have 
\begin{align}
\label{eq:posivity_right_eigvect}
    \vect{v}^{(r)}_1  \succ \vect{0}_K \, . 
\end{align}
Using the identity $\vect{a} \odot \vect{b} \equiv \mat{D}_{\vect{a}} \vect{b}$ and the definition of the matrix $\matGam$ of Eq.~\eqref{eq:def_matGam}, we can write the QVE \eqref{eq:QVE} as:
\begin{align}
\label{eq:property_right_eigvect}
    (\mat{D}_{\vect{g}(1)} \matOme) \vect{v}^{(r)}_1 = \vect{v}^{(r)}_1 \, .
 \end{align}
In other words, we have explicitly shown that $1$ is an eigenvalue of $\mat{D}_{\vect{g}(1)} \matOme$ (since  $\vect{v}^{(r)}_1 \neq \vect{0}_K$), which gives by Prop.~\ref{prop:eq_outlier} that $1$ is also the position of a limiting outlier of $\rdmmat{\Tilde{Y}}/\sqrt{N}$. 
\jump
It is worth mentioning the following result, which we will use later on.
\begin{lemma}
\label{lem:1_simple_eig_DgMatOme}
   For $\lambda_1(\matOme) >1$,  $1$ is  simple and the highest eigenvalue of $\mat{D}_{\vect{g}(1)} \matOme$.
\end{lemma}
\begin{proof}
    From Lem.~\ref{lem:g1_SNR_g_1}, the matrix $\mat{D}_{\vect{g}(1)} \matOme$ has positive entries and so the Perron-Frobenius theorem yields that its highest eigenvalue is simple with an associated positive (right) eigenvector.  Since we show that the eigenvector associated with the eigenvalue $1$ satisfies this positivity condition (see Eq.~\eqref{eq:posivity_right_eigvect}), we have the desired result for $1$.
\end{proof} 
\jump 
Note that if we left-multiply the QVE by $\matGam$ and denote by 
\begin{align}
\label{eq:def_left_eigvect}
    \vect{v}^{(l)}_1 := \matOme \mat{D}_{(\vect{\rho}^{\odot 1/2})} (\vect{1}_K  - \vect{g}(1)) 
\end{align}
we obtain the left eigenvector:
\begin{align}
\label{eq:property_left_eigvect}
    (\mat{D}_{\vect{g}(1)} \matOme)^{\top} \vect{v}^{(l)}_1 = \vect{v}^{(l)}_1 \, .
 \end{align}
Since $\lambda_r <1$, for $\lambda_1(\matOme) >1$, we have that the outlier does indeed separate from the bulk. 
\jump
To conclude, from Lem.~\ref{lem:1_simple_eig_DgMatOme}, we have $\lambda_1( \mat{D}_{\vect{g}(1)} \matOme ) =1$, the use of the monotonous behavior of $\vect{g}(.)$ from Lem.~\ref{lem:pos_and_decreasing_sol_QVE} leads to $\lambda_1( \mat{D}_{\vect{g}(\lambda)} \matOme)  <1$ for any $\lambda >1$, from which we deduce that $1$ is necessary the leading outlier of $\rdmmat{\Tilde{Y}}/\sqrt{N}$ by Prop.~\ref{prop:eq_outlier}.

\subsection{Outline of the Proof for the Overlap Vector}
To compute the limiting value of the overlap vector $\vect{\mu}$, let us first notice that this vector can be written in terms of eigenmatrix $\mat{V}$ of Prop.~\ref{prop:decomposition_Z} as
\begin{align}
    \vect{\mu} &= \mat{V}^{\top} \vect{u}_1 \, , 
\end{align}
where we recall that $\vect{u}_1$ is the eigenvector associated with the top eigenvalue  of $\rdmmat{\Tilde{Y}}/\sqrt{N}$.
\jump
Next, the \emph{resolvent} of  $\rdmmat{\Tilde{Y}}/\sqrt{N}$ is defined for any $z \in \mathbb{C} \setminus \mathrm{Spec}( \rdmmat{\Tilde{Y}}/\sqrt{N})$ by  
\begin{align}
     \mat{G}_{\frac{\rdmmat{\Tilde{Y}}}{\sqrt{N}}}(z) &:= \Big (z - \frac{\rdmmat{\Tilde{Y}}}{\sqrt{N}} \Big )^{-1} \, ,\\
     &= \sum_{i=1}^{N} \Big (z- \lambda_i \big( \frac{\rdmmat{\Tilde{Y}}}{\sqrt{N}} \big) \Big)^{-1} \vect{u}_i \vect{u}_i^{\top} \, . 
\end{align}
As $ \lambda_1(\rdmmat{\Tilde{Y}}/\sqrt{N}) \xrightarrow[N \to \infty]{\as} 1 $ and separates from the bulk  if $\lambda_1(\matOme) >1$ and is  also almost surely simple by Lem.~\ref{lem:1_simple_eig_DgMatOme} and Prop.~\ref{prop:eq_outlier}, we have the following identity:
\begin{lemma}
\label{lem:overlap_from_resolvent}
For $\lambda_1(\matOme) >1$, we have:
\begin{align}
\vect{\mu} \vect{\mu}^{\top} &=\underset{z \to 1}{\lim}  \, (z-1) \, \matV^{\top}  \mat{G}_{\frac{\rdmmat{\Tilde{Y}}}{\sqrt{N}}}(z)  \matV  \, .
\end{align}
\end{lemma}
From this identity, we first obtain the following intermediate result: 

\begin{proposition}
\label{prop:overlap_propto_eigenvector}
For $\lambda_1(\matOme) >1$, we have:
    \begin{align}
    \vect{\mu} \vect{\mu}^{\top}   &\xrightarrow[N \to \infty]{\as} \frac{  \vect{v}^{(r)}_1  ({\vect{v}^{(r)}_1})^{\top} }{-\phi_1'(1)}  \, ,
\end{align}
where $ \vect{v}^{(r)}_1$ is given by Eq.~\eqref{eq:def_right_eigvect}, $\phi_1(\lambda)$ is the top eigenvalue of $\mat{D}_{\vect{g}(\lambda)} \matOme$ and  $\phi_1'(.)$ is the derivative of $\phi_1(.)$ with respect to $\lambda$. 
\end{proposition}
\begin{proof}
    The proof of this result is given in Appendix \ref{sec:proof_Overlap_propto_right_eigenvector}. It relies first on using Woodbury identity to express the resolvent of $\rdmmat{\Tilde{Y}}/\sqrt{N}$ in terms of the one of $\rdmmat{X}/\sqrt{N}$, then using deterministic equivalent to simplify the expression and eventually identify the limit in Lem.~\ref{lem:overlap_from_resolvent} as the derivative of $\phi_1(.)$ evaluated at one.
\end{proof}
To finish the proof, we show that
\begin{proposition}
\label{prop:value_of_norm_constant_overlap}
For $\lambda_1(\matOme) >1$, with $\phi_1'(1)$ as in Prop.~\ref{prop:overlap_propto_eigenvector} we have
\begin{align}
    -\phi_1'(1) &= C \, ,
\end{align}
where $C$ is given by Eq.~\eqref{eq:def_C}.
\end{proposition}
\begin{proof}
    The proof of this result is given in Appendix \ref{sec:proof_value_of_constant_of_proportionality} and relies on perturbation theory for the top eigenvalue of $\mat{D}_{\vect{g}(1 + \epsilon)} \matOme$.
\end{proof}
Thanks to the definition \eqref{eq:def_right_eigvect} of ${\vect{v}^{(r)}_1}$, one has the desired result for the overlap. 
\jump
The case $\lambda_1(\matOme) \leq 1$ where there is no outlier, follows from standard results in RMT, see e.g. \cite{benaych-georges_eigenvalues_2011}.

\section{Conclusion}
In this work, we studied the behavior of a spiked block-Wigner matrix model given by Eq.~\eqref{eq:def_init_mat_tilde_Y}. Our first contribution is the proof of a phase transition for the behavior of the top eigenvalue of this model, generalizing the seminal work of \cite{BBP} to inhomogeneous problems, proving a conjecture stated in \cite{pak2023optimal}, and providing a sharp optimal method for detection for the inhomogenous spiked model \cite{behne2022fundamental,alberici2021multi,alberici2022statistical,AJFL_inhomo}. Our results could also be used as spectral start \cite{mondelli2021approximate} for the AMP algorithms of \cite{pak2023optimal}, making them optimal in terms of MMSE for these problems in the class of iterative algorithm \cite{AJFL_inhomo}.

Our second contribution is a sharp characterization of the overlap between the associated top eigenvector of this matrix and the original signal vector, undergoing a similar phase transition. While we have, for simplicity, considered the setting where the signal is a rank-one matrix,  our computation can be carried to an arbitrarily fixed rank-$R$ case.

An interesting venue for future work is to extend this computation to a general setting of general variance-profile shape (not necessarily of block type) and we leave this problem for future work. Finally, it would also be interesting to understand the fluctuations of the top eigenvalue around its limiting value in this inhomogeneous one, which are very well understood for a homogeneous problem.

\jump

\paragraph{Fundings and Acknowledgments -}
\label{sec:fund_and_acknow} 
The authors would like to thank Alice Guionnet and Aleksandr Pak for fruitful discussions at the early stage of this project and Brice Huang and Mark Selke for discussions on their work \cite{huang_strong_2023}.   We also acknowledge funding from the Swiss National Science Foundation grant SNFS OperaGOST  (grant number $200390$) and from the ERC Project LDRAM: ERC-2019-ADG Project 884584.

\bibliographystyle{amsplain}
\bibliography{bib_abs}

\newpage
\appendix
\onecolumn 

\section{Sketch of Proof of the Derivation of the Quadratic Vector Equation}
\label{sec:proof_QVE}
We briefly sketch the proof for the derivation of the Prop.~\ref{prop:QVE} and refer the reader interested in the complete rigorous proof to the series of work \cite{ajanki_singularities_2017,ajanki_universality_2017,ajanki_quadratic_2019}.

Let $z \in \mathbb{H}_-$ and denote by $\mat{G}_{\rdmmat{X}/\sqrt{N}}(z) := (z- \rdmmat{X}/\sqrt{N})^{-1}$ the resolvent of $\rdmmat{X}/\sqrt{N}$. From Schur complement formula, we have for any $i \in [N] := \{1,\dots,N\}$:
\begin{align}
\label{eq:Schur_complement}
    \frac{1}{G_{ii}(z)} = z - X_{ii} - \langle \rdmvect {\Vec{X}}_{\bullet i}, \mat{G}^{(-i)}(z)  \rdmvect{\Vec{X}}_{\bullet i} \rangle \, , 
\end{align}
where to ease notations we denoted by $G_{ij} \equiv ( \mat{G}_{\rdmmat{X}/\sqrt{N}} )_{ij}$, $X_{ij} \equiv (\rdmmat{X}/\sqrt{N})_{ij}$, ${\Vec{X}}_{\bullet i} \in \mathbb{R}^{N-1}$ is the $i$-th column of $\rdmmat{X}/\sqrt{N}$ without the $(i,i)$ entry and $\mat{G}^{(-i)}(z)$ is the resolvent of the $\big( (N-1) \times (N-1) \big)$ matrix $\rdmmat{X}^{-(i)}/\sqrt{N}$ obtained by removing the $i$-th row and $i$-th column of  $\rdmmat{X}/\sqrt{N}$.

By definition of $ \rdmmat{X}/\sqrt{N}$, we have   $ X_{ii} = H_{ii} \sqrt{(\mat{\Sigma})_{ii}}/\sqrt{N}  -   (\mat{\Sigma} \vect{1})_{ii}/N$, with $ H_{ii} \sim \mathcal{N}(0,1)$. Since the first term is of order $\mathcal{O}(N^{-1/2})$,  for large $N$ we have approximately 
\begin{align}
    X_{ii} &\approx  -\frac{1}{N} \sum_{j=1}^{N} (\mat{\Sigma})_{ij}  =  -\frac{1}{N} \sum_{k=1}^K \sum_{j \in B_k}  (\mat{\Sigma})_{ij} \, ,
\end{align}
and by assumption $ (\mat{\Sigma})_{ij} = s_{kl}$ for any  $i \in B_k,j \in B_l$, such that for any $k$, the  $X_{ii}$ within $B_k$  are (approximately) equal and given by:
\begin{align}
    X_{ii} &\approx - \sum_{l} s_{kl} \rho_l     \, .
\end{align}
The quadratic term in Eq.~\eqref{eq:Schur_complement} is given by $ \langle \rdmvect {\Vec{X}}_{\bullet i}, \mat{G}^{(-i)}(z)  \rdmvect{\Vec{X}}_{\bullet i} \rangle = \sum_{j_1,j_2 \neq i} X_{j_1,i} X_{j_2,i}  (\mat{G}^{(-i)}(z))_{ij}$ and can be shown to concentrate to its average value. Since $X_{j_1,i}$ and $X_{j_2,i}$ are independent and centered (since $j_1,j_2 \neq i$), this averaged sum is equal to the sum over the $(N-1)$ term such that $j_1 = j_2 \neq i$. Furthermore, for large $N$, replacing  $\mat{G}^{(-i)}(z)$  by $\mat{G}_{\rdmmat{X}/\sqrt{N}}(z) $ and  adding the term $j_1 = i$ in the sum only adds a negligible contribution to the sum such that, one has approximately:
\begin{align}
    \langle \rdmvect {\Vec{X}}_{\bullet i}, \mat{G}^{(-i)}(z)  \rdmvect{\Vec{X}}_{\bullet i} \rangle \approx \frac{1}{N} \sum_{j=1}^{N}  (\mat{\Sigma})_{ij}  G_{jj}(z) = \frac{1}{N} \sum_{k=1}^{K} \sum_{j \in B_k}  (\mat{\Sigma})_{ij}  G_{jj}(z) \, ,
\end{align}
and Eq.~\eqref{eq:Schur_complement} is thus approximately given by
\begin{align}
\label{eq:Approx_Schur}
    \frac{1}{G_{ii}(z)} \approx  z + \sum_{l=1}^{K} s_{kl} \rho_l - \frac{1}{N} \sum_{j=1}^{N} (\mat{\Sigma})_{ij} G_{jj}(z) \, .
\end{align}
Using again the block structure of $\mat{\Sigma}$, one deduces from Eq.~\eqref{eq:Approx_Schur} the approximate identity $G_{ii}(z) \approx G_{jj}(z)$ if $i,j \in B_k$. Furthermore, $N^{-1} \mathrm{Tr} \, \mat{G}_{\rdmmat{X}/\sqrt{N}}(z) = N^{-1} \sum_{i=1}^N G_{ii}(z)  \to g_{\mu_X}(z)$, and if we denote by $g_k(z) := \lim_{N\to \infty} 1/|B_k| \sum_{i \in B_k} G_{ii}(z) \approx G_{ii}(z)$ for $i \in B_k$, we have $g_{\mu_X} = \sum_{k} \rho_k g_k$ with for any $k \in [K]$, $g_k$ is solution of
\begin{align}
    \frac{1}{g_k} = z  +\sum_{l=1}^{K} s_{kl} \rho_l - \sum_{l=1}^{K} s_{kl} \rho_l g_l \, ,
\end{align}
which is nothing else than the scalar version of the QVE of Prop.~\ref{prop:QVE}.

\section{Proof of  Prop.~\ref{prop:condition_good_sol_QVE_real_line} (Stietljes transform on the Real Line and Positivity Condition for Solution of Associated Linear System)}
\label{sec:proof_Linear_System_QVE_real_line}

For $z \in \mathbb{H}_-$  let's rewrite the QVE of Eq.~\eqref{eq:QVE} with solution $\vect{g} \in (\mathbb{H}_+)^K$ as: 
\begin{equation}
     z \vect{1}_K = \vect{f}( \vect{g}(z)) \, ,
\end{equation}
 where $\vect{f} : (\mathbb{H}_+)^K \to  (\mathbb{H}_-)^K$ is defined by
 \begin{equation}
      \vect{f}(\vect{g}) := \vect{g}^{\odot -1} + \matGam(\vect{g} - \vect{1}_K) \, .
 \end{equation}
Clearly, the function $\vect{f}(.)$ is continuous and in fact $C^1$ and invertible on $ (\mathbb{H}_+)^K$ with Jacobian given by 
\begin{align}
\label{eq:def_jacobian}
    (\nabla_{\vect{g}}  \vect{f})(\vect{g}) = - \mat{D}_{\vect{g}}^{-2} + \matGam \, ,
\end{align}
and the $g_k(z) := \int (z - \lambda')^{-1} \mathrm{d}\mu_k(\lambda')$ are analytical on $\mathbb{C} \setminus \overline{\mathrm{Supp}(\mu_X)}$ for any $k \in [K]$ since $\mathrm{Supp}(\mu_k) \subseteq \mathrm{Supp}(\mu_X) $ for any $k \in [K]$, from \cite{ajanki_quadratic_2019}.

 Fix any $\lambda \in \mathbb{R} \setminus \overline{\mathrm{Supp}(\mu_X)}$, as we approach $\lambda$ from the lower complex plane $z \to \lambda$, $\vect{g}(z)$ approaches its analytical continuation $\vect{g}(\lambda)$ from the upper complex plane, which means that   $\vect{g}'(\lambda)$ is well defined and belongs also to $(\mathbb{H}_+)^K$.  By the inverse function theorem, if we set $\vect{\Tilde{y}} = \vect{g}'(\lambda)$, we must have the equality:  
\begin{align}
    - \vect{1}_K =   (\nabla_{\vect{g}}  \vect{f})(\vect{g}(\lambda)) \vect{\Tilde{y}} \, ,
\end{align}
with the condition $\vect{\Tilde{y}} \in (\mathbb{H}_+)^K$. Taking the imaginary part of this equation gives the desired result by setting $y_i = \mathfrak{Im} (\Tilde{y}_i) $  since for $\lambda \in \mathbb{R} \setminus \overline{\mathrm{Supp}(\mu_X)}$, $\vect{g}(\lambda) \in \mathbb{R}^K$ and hence we also have $(\nabla_{\vect{g}}  \vect{f})(\vect{g}(\lambda)) \in \mathbb{R}^{K \times K}$. Conversely, for any other solutions $\vect{\Tilde{g}} \in (\mathbb{R}_*)^K$  of the QVE at $z=\lambda$, one cannot have $\vect{y} \succeq \vect{0}_K$ since $\vect{f}$ and $(\nabla_{\vect{g}}  \vect{f})$ are well defined on $(\mathbb{C}_*)^K$ and $\vect{g} \in (\mathbb{H}_+)^K$ is the unique solution of the QVE on the upper complex plane. 

\section{Proof of Lem.~\ref{prop:eigenvalue_condition_good_sol_QVE_real_line} (Eigenvalue Condition for Positive Solution of the Linear System) }
\label{sec:proof_positivity_of_solution_linear_system}
By multiplying Eq.\eqref{eq:condition_good_solution_QVE} by $\mat{D}_{\vect{\Tilde{g}}}^2$, we can re-write Eq.~\eqref{eq:condition_good_solution_QVE} as
\begin{align}
    (\mat{I}_K - \mat{D}_{\vect{\Tilde{g}}}^2 \matGam) \vect{y} = \vect{\Tilde{g}}^{\odot 2} \, ,
\end{align}
since $\mat{D}_{\vect{\Tilde{g}}}^2 \matGam$ and $\mat{D}_{\vect{\Tilde{g}}}^2 \matOme$ are similar by definition of $\matGam$, we have $ \lambda_1(\mat{D}_{\vect{\Tilde{g}}}^2 \matGam) =  \lambda_1(\mat{D}_{\vect{\Tilde{g}}}^2 \matOme)$. To conclude, it is enough to prove the following result. 
\begin{proposition}
\label{prop:positivity_linear_system}
    Let $\mat{Q} \in (\mathbb{R}_+)^{d \times d}$ with real eigenvalues,  $\vect{r} \succ \vect{0} \in \mathbb{R}^d$ then the linear system $(\mat{I} - \mat{Q}) \vect{z} = \vect{r}$ admits a unique positive solution $\vect{z} \succ \vect{0}$ if and only all eigenvalues of $\mat{Q}$ are lower than one. 
\end{proposition}

To prove this proposition, following \cite{plemmons_m-matrix_1977} and Chap.~6 of \cite{berman_nonnegative_1994} (see also Chap.~1 of \cite{saad_iterative_2003}), we first introduce the two sets of matrix as:
\begin{defn}[$Z$- and $M$-matrices]
\label{def:Zmat_and_Mmat}
    We say that a square matrix $\mat{A}$ is a $Z$-matrix and write $\mat{A} \in \mathfrak{Z}$ if all its off-diagonal elements are non-positive. If in addition, all its eigenvalues are positive we say that  $\mat{A}$ is $M$-matrix and write $\mat{A} \in \mathfrak{M}$.
\end{defn}
\begin{itemize}
    \item Let's first consider the case $\lambda_1(\mat{Q}) < 1$. In this case, one can immediately check that $\mat{M}:= \mat{I}_K -\mat{Q} \in \mathfrak{M}$ and one has the following lemma for $M$-matrices.
    \begin{lemma}[Inverse Positivity of $M$-matrices]
    \label{lem:Inverse_Positivity}
    If $\mat{A} \in \mathfrak{M}$  and is non-singular then $\mat{A}^{-1}$ has all it entries non-negative.
    \end{lemma}
    \begin{itemize}
        \item  \begin{proof}[Proof of Lem.~\ref{lem:Inverse_Positivity}] See Chap.~6 of \cite{berman_nonnegative_1994} for the complete proof. We briefly outline the main idea for completeness. If we denote by $\vect{a}:= (A_{11},\dots, A_{dd})$ the vector of diagonal elements, we have $\vect{a} \succ \vect{0}$. We can then write $\mat{A} = \mat{D}_{\vect{a}} (\mat{I} - \mat{B}) $ with $\mat{B} = (\mat{I} - \mat{D}_{\vect{a}}^{-1} \mat{A})$ and since $(\mat{I}-\mat{B})^{-1} = \mat{A}^{-1}  \mat{D}_{\vect{a}}$ can be shown to be entrywise non-negative and $\mat{D}_{\vect{a}}^{-1}$ is positive,  we get the desired result for $\mat{A}^{-1}$.
    \end{proof}
    \end{itemize} 
    As a consequence of this lemma, we have with the notations of Prop.~\ref{prop:positivity_linear_system}, $\vect{y} = \mat{M}^{-1} \vect{r}  \succ \vect{0}_K $.
    \item Next for $\lambda_1(\mat{Q}) > 1$, the same matrix $\mat{M}:= \mat{I}_K -\mat{Q} \in \mathfrak{Z} \setminus \mathfrak{M}$ (since it has at least one negative eigenvalue). The Perron-Frobenius theorem for $Z$-matrices writes 
    \begin{lemma}[Perron-Frobenius $Z$-matrices]
    \label{lem:Perron-Frobenius_matZ}
    If $\mat{A} \in \mathfrak{Z}$, there exists a real number $\sigma$ and a positive vector $\vect{u} \succ \vect{0}$ such that
    \begin{enumerate}
        \item[(a)]  $\mat{A}^{\top} \vect{u} = \sigma \vect{u}$, 
        \item[(b)] $\mathfrak{Re}(\lambda) > \sigma$ if $\lambda$ is any characteristic root of $\mat{A}$ with $\lambda \neq \sigma$.
    \end{enumerate}
    \end{lemma}

        \begin{itemize}
        \item  \begin{proof}[Proof of Lem.~\ref{lem:Perron-Frobenius_matZ}]
        see Chap.~6 of \cite{berman_nonnegative_1994}. The proof is similar to the standard Perron-Frobenius theorem for entrywise positive matrix.
                \end{proof}
        \end{itemize} 
   
     As $ 1-\lambda_1(\mat{Q})$ is the lowest eigenvalue of $\mat{M}^{\top}$, this Perron-Frobenius lemma implies that the corresponding eigenvector is positive $\vect{u} \succ \vect{0}$. In particular, we have $\langle - \vect{u} , \vect{1} \rangle < 0$ and $\mat{M}^{\top} (- \vect{u}) = (\lambda_1(\mat{Q}) -1) \vect{u} \succeq \vect{0}$, which allows us to conclude thanks to Farka's lemma:
    \begin{lemma}[Farka's lemma]
    Let $\mat{A} \in \mathbb{R}^{d \times r}$, $\vect{b} \in \mathbb{R}^r$ then either
    \begin{enumerate}
        \item[(a)] there exists $\vect{y} \in \mathbb{R}^r$ solution of $\mat{A} \vect{y} = \vect{b}$ with $\vect{y} \succeq \vect{0}$, 
        \item[(b)] or there exists $\vect{z} \in \mathbb{R}^d$ such that  $\mat{A}^{\top} \vect{z} \succeq \vect{0}$ and $\langle \vect{z} , \vect{b} \rangle <0$. 
    \end{enumerate}
    \end{lemma}

    \begin{itemize}
        \item \begin{proof}[Proof of Farka's lemma]
    see Chap.~5 of \cite{boyd_convex_2004} where the proof follows from linear programming duality.
    \end{proof}
    \end{itemize}
Indeed, we have explicitly constructed  $\vect{z} = - \vect{u}$ such that condition (b) holds for the linear system $ ( \mat{I}_K -\mat{Q}) \vect{y} = \vect{r} $ and as such, it does not admit a positive solution $\vect{y}$ for $\lambda_1(\mat{Q}) > 1$.
\end{itemize}
This concludes the proof of Prop.~\ref{prop:positivity_linear_system} since if $\lambda_1(\mat{Q})=1$ then the matrix $\mat{M} = \mat{I} - \mat{Q}$ is singular and one has no more uniqueness for the possible solution of the associated linear system, see e.g. Chap.~1 of \cite{saad_iterative_2003}. 

\section{Proof of Prop.~\ref{prop:bound_rightmost_edge} (Upper Bound of the Rightmost Edge)}
\label{sec:proof_bound_rightmost_edge}

\begin{proposition}[Condition for being an edge]
    If $\lambda_{\star}$ is the position of an edge then 
    \begin{align}
    \label{eq:condition_edge}
        \lambda_{\star} \vect{1}_K = \vect{f}(\vect{g}_{\star}) \, ,
    \end{align}
    for some $\vect{g}_{\star} \in (\mathbb{R}_*)^K$  such that $(\nabla_{\vect{g}} \vect{f})(\vect{g}_\star)$ is singular. 
\end{proposition}

\begin{proof}
This follows from Thm.~2.6 of \cite{ajanki_singularities_2017}.
\end{proof}

\begin{lemma}
\label{lem:condition_singular_jacobian}
     Let $\vect{g} \in (\mathbb{R}_*)^K$, then $(\nabla_{\vect{g}} \vect{f})(\vect{g})$ is singular if there exists $\vect{w} \succ \vect{0}_K$ such that
     \begin{align}
     \label{eq:condition_singularity}
       \vect{w}^{\top}  (\mat{D}_{\vect{g}}^2 \matGam)   = \vect{w}^{\top} \, .
     \end{align}
\end{lemma}
\begin{proof}
First since $\vect{g} \in (\mathbb{R}_*)^K$,  $(\nabla_{\vect{g}} \vect{f})(\vect{g})$ is singular if $ \mat{D}^2_{\vect{g}} (\nabla_{\vect{g}} \vect{f})(\vect{g})$ is, which means from  definition \eqref{eq:def_jacobian} that there exists a $\vect{w} \neq \vect{0}_K$ such that Eq.~\eqref{eq:condition_singularity} holds. Recalling Def.~\ref{def:Zmat_and_Mmat}, we have $\mat{M} = \mat{I}_K -   (\mat{D}_{\vect{g}}^2 \matGam) \in \mathfrak{Z}$ from which we deduce the positivity of $\vect{w}$ by Lem.~\ref{lem:Perron-Frobenius_matZ}.
\end{proof}
\jump For $\lambda_{\star}$ the position of any edge, let's denote by  $\vect{w}_{\star}$ the corresponding left eigenvector in Lem.~\ref{lem:condition_singular_jacobian}. If we left multiply Eq.~\eqref{eq:condition_edge} by $\vect{w}_{\star}^{\top} \mat{D}_{\vect{g}_{\star}}^2$  we have 
\begin{align}
    \lambda_{\star} \langle \vect{w}_{\star} , \vect{g}_{\star}^{\odot 2} \rangle = \langle \vect{w}_{\star} , \vect{g}_{\star} \rangle +  \langle \vect{w}_{\star} ,  (\mat{D}_{\vect{g}_{\star}}^2 \matGam) (\vect{g}_{\star} - \vect{1}_K) \rangle \,  .
\end{align}
Using Eq.~\eqref{eq:condition_singularity} this reads
\begin{align}
    \lambda_{\star}= \frac{\langle \vect{w}_{\star}, 2 \vect{g}_\star - \vect{1}_K \rangle }{\langle \vect{w}_{\star} , \vect{g}_{\star}^{\odot 2} \rangle} \, ,
\end{align}
and one gets the desired result thanks to $\vect{w}_{\star} \succ \vect{0}_K$ and the identity
\begin{align}
     \vect{1}_K + \vect{g}_{\star}^{\odot 2} -2 \vect{g}_{\star} =   (\vect{1}_K - \vect{g}_\star)^{\odot 2}  \succeq \vect{0}_K \, ,
\end{align}
with equality if and only if $\vect{g}_\star = \vect{1}_K$ which is only attainable at $\lambda_1(\matOme) =1$ since $\vect{g}(1) = \vect{1}_K $ for $\lambda_1(\matOme) \leq 1$ from Cor.~\ref{cor:behavior_g1} and $ \mat{D}^2_{\vect{g}} \,  (\nabla_{\vect{g}} \vect{f})(\vect{1}_K)$ is not singular for $\lambda_1(\matOme) < 1$ from App.~\ref{sec:proof_positivity_of_solution_linear_system}, and is clearly singular for $\lambda_1(\matOme) =1$.

\section{Proof of Prop.~\ref{prop:decomposition_Z} (Eigendecomposition of the Small-Rank Matrix \texorpdfstring{$\mat{Z}$)}{}}
\label{sec:proof_Eigendecomposition_matZ}
By definition of $\mat{Z}/\sqrt{N}$, if we denote for $k \in \{1,\dots, K\}$, the vector $\vect{x}_k$ where $(\vect{x}_k)_i := x_i$ for $x \in B_k$ and $(\vect{x}_k)_i = 0$ for $i \notin B_k$ as defined in \eqref{eq:def_overlap}, we have:
\begin{align}
    \frac{\mat{Z}}{\sqrt{N}} = \frac{1}{N} \sum_{k,l}^{K} s_{kl} \, \vect{x}_k \vect{x}_l^{\top} \, .
\end{align}
For $N$ large enough, the norms of the $\vect{x}_k$ are almost surely non-zero since the $x_i$'s are iid with variance one and we can normalize the vectors by their norms
\begin{align}
    \frac{\mat{Z}}{\sqrt{N}} 
    =
    \sum_{k,l}^{K} s_{kl} \frac{\sqrt{|B_k|  |B_l|}}{N} 
    \,
    \frac{ \| \vect{x}_k \|  \| \vect{x}_l \| }{\sqrt{|B_k|  |B_l|  }} 
    \, 
    \frac{\vect{x}_k}{\| \vect{x}_k \|} \Big( \frac{\vect{x}_l}{\| \vect{x}_l \|} \Big)^{\top} \, ,
\end{align}
where we have dropped the dependency in $N$ for $B_k(N)$ for clarity. By assumption  we have $|B_k|/N \to \rho_k \in (0,1)$  and by the law of large numbers, we  have $\| \vect{x}_k \|/\sqrt{|B_k|} \xrightarrow[N \to \infty]{\mathrm{a.s}} 1$. Since $(\matOme)_{kl} = s_{kl} \sqrt{\rho_k} \sqrt{\rho_l}$, the remainder term converges almost surely to zero, and hence the operator norm of the associated matrix $\rdmmat{E}_K$ converges also almost surely to zero by matrix norm equivalence and the fact its dimension $K$ is fixed, thus we have the desired result. 

\section{Proof of Prop.~\ref{prop:eq_outlier} (Equation for Outliers)}
\label{sec:proof_Equation_Outliers}
Let $\lambda$ not in the spectrum of $\rdmmat{X}/\sqrt{N}$ such that $ \mat{G}_{\rdmmat{X}/\sqrt{N}}(\lambda) := ( \lambda- \rdmmat{X}/\sqrt{N})^{-1}$ is well defined, then from the decomposition of Eq.~\eqref{eq:YeqXplusZ}, we have the identity
\begin{align}
    \Big(\lambda - \frac{\rdmmat{\Tilde{Y}}}{\sqrt{N}} \Big) = \Big(\lambda - \frac{\rdmmat{X}}{\sqrt{N}} \Big) \Big( \mat{I} - \mat{G}_{\frac{\rdmmat{X}}{\sqrt{N}}}(\lambda)  \mat{Z}  \Big) \, .
\end{align}
Taking the determinant of this equation with the eigendecomposition of Prop.~\ref{prop:decomposition_Z} and cyclicity of the determinant, we get that the characteristic polynomial of $\rdmmat{\Tilde{Y}}/\sqrt{N}$ is given by:
\begin{align}
    \mathrm{det} \Big(z - \frac{\rdmmat{\Tilde{Y}}}{\sqrt{N}} \Big) = \mathrm{det} \Big(z - \frac{\rdmmat{X}}{\sqrt{N}} \Big) \Big( \mat{I}_K - \matV^{\top} \mat{G}_{\frac{\rdmmat{X}}{\sqrt{N}}}(\lambda) \matV \mat{\Tilde{\Omega}}_K \Big) \, , 
\end{align}
where we have a denoted by $\mat{\Tilde{\Omega}}_K := \matOme + \rdmmat{E}_K$. Following \cite{benaych-georges_eigenvalues_2011,benaych-georges_singular_2012} as $N \to \infty$, an outlier of $\rdmmat{\Tilde{Y}}/\sqrt{N}$ separates from the bulk if it is a zero of the following  limit of the secular function
\begin{align}
\label{eq:def_secular_function}
    h(\lambda) := \lim_{N \to \infty} \mathrm{det} \Big( \mat{I}_K - \matV^{\top} \mat{G}_{\frac{\rdmmat{X}}{\sqrt{N}}}(\lambda) \matV \mat{\Tilde{\Omega}}_K \Big) \, ,
\end{align}
provided this limit is well-defined. 
\jump
To tackle this limit, we will use deterministic equivalents for the resolvent.
\begin{lemma}[Anisotropic Deterministic Equivalent]
\label{lem:deter_equiv}
Let $z \in \mathbb{C} \setminus \mathrm{Supp}(\mu_X)$, $\vect{w}_1,\vect{w}_2$ be two sequences of vectors independent of $\rdmmat{X}$ such that for $i=1,2$, $\| \vect{w}_i \| \xrightarrow[N \to \infty]{\mathrm{a.s}} C_i >0$ , and $\overline{\mat{G}(z)} $ the  $(N \times N)$ diagonal matrix
\begin{align}
    \overline{\mat{G}(z)} := \begin{pmatrix}
     g_1(z) \mat{I}_{|B_1|} &  &    \\
      & \ddots  &  \\ 
     &   &   g_K(z) \mat{I}_{|B_K|} 
\end{pmatrix}   \, ,
\end{align}
 where the $g_k(.)$ are the solutions of the QVE of Prop.~\ref{prop:QVE}, then we have
 \begin{equation}
 \begin{split}
     \Big | \langle \vect{w}_1 , \mat{G}_{\frac{\rdmmat{X}}{\sqrt{N}}}(z) \vect{w}_2 \rangle - 
       \langle \vect{w}_1, \overline{\mat{G}(z)} \vect{w}_2 \rangle
    \Big |  \xrightarrow[N \to \infty]{\mathrm{a.s}} 0.
\end{split}
 \end{equation}
\end{lemma}
\begin{proof}
    This is a direct consequence of the anisotropic local law of Thm.~1.13 in \cite{ajanki_universality_2017}. 
\end{proof}
\jump Since $(\matV^{\top} \mat{G}_{\rdmmat{X}/\sqrt{N}}(\lambda) \matV)_{kl} = \langle \vect{v}^{(k)} , \mat{G}_{\rdmmat{X}/\sqrt{N}}(\lambda)   \vect{v}^{(l)} \rangle$, the use of Lemma \ref{lem:deter_equiv} with the block-structure of the  $\vect{v}^{(k)}$ from Prop.\ref{prop:decomposition_Z} implies 
 \begin{equation}
 \label{eq:deter_equiv_VGV}
\matV^{\top} \mat{G}_{\rdmmat{X}/\sqrt{N}}(\lambda) \matV \xrightarrow[N \to \infty]{\mathrm{a.s}} \mat{D}_{\vect{g}(\lambda)} \, ,
 \end{equation}
and since also $\mat{\Tilde{\Omega}}_K \xrightarrow[N \to \infty]{\mathrm{a.s}} \matOme$ both in operator norm and entrywise  by Prop.~\ref{prop:decomposition_Z}, by continuity of the determinant this leads to
\begin{align}
    h(\lambda) \overset{\mathrm{a.s}}{=}  \det \left( \mat{I}_K - \mat{D}_{\vect{g}(\lambda)} \, \mat{\Omega}_K  \right) \, ,
\end{align}
which concludes the proof.

\section{Proof of Prop.~\ref{prop:overlap_propto_eigenvector} (The Overlap Vector is Proportional to the Right top Eigenvector \texorpdfstring{$\vect{v}_1^{(r)}$}{})}
\label{sec:proof_Overlap_propto_right_eigenvector}

Let $z$ such that $(z \mat{I} - \rdmmat{X}/\sqrt{N})$ and  $(z \mat{I} - \rdmmat{\Tilde{Y}}/\sqrt{N})$ are invertible. By Woodbury matrix identity (see e.g. Chap.~3 of \cite{meyer_matrix_2023}), we can relate the resolvent of $ \rdmmat{\Tilde{Y}}/\sqrt{N}$ to the one of $ \rdmmat{X}/\sqrt{N}$ by
\begin{align}
    \mat{G}_{\frac{\rdmmat{\Tilde{Y}}}{\sqrt{N}}}(z) = \mat{G}_{\frac{\rdmmat{X}}{\sqrt{N}}}(z) + \mat{G}_{\frac{\rdmmat{X}}{\sqrt{N}}}(z)  \matV \matOme \big( \mat{I}_K - \matV^{\top}    \mat{G}_{\frac{\rdmmat{X}}{\sqrt{N}}}(z) \matV \matOme \big)^{-1}  \, \matV^{\top}  \mat{G}_{\frac{\rdmmat{X}}{\sqrt{N}}}(z) \, .
\end{align}
If we both left-multiply this equation by $\matV^{\top}$ and right-multiply by $\matV$, we get 
\begin{align}
\label{eq:VGV_tildeG}
    \matV^{\top} \mat{G}_{\frac{\rdmmat{\Tilde{Y}}}{\sqrt{N}}}(z) \matV = \mat{\tilde{G}} + \mat{\tilde{G}} \matOme( \mat{I}_K - \mat{\tilde{G}} \matOme)^{-1}  \mat{\tilde{G}} \, ,
\end{align}
where we introduced the matrix
\begin{align}
    \mat{\tilde{G}} := \matV^{\top} \mat{G}_{\rdmmat{X}/\sqrt{N}}(z) \matV  \, ,
\end{align}
to ease notations. 

\jump For $\lambda_1(\matOme) >1$, as $N \to \infty$, the spectrum  of $\rdmmat{X}/\sqrt{N}$ contains almost surely no eigenvalue in a $\epsilon$-neighborhood of one and thus $ \lim_{N \to \infty} \lim_{z \to 1} (z-1) \mat{\tilde{G}} \overset{\mathrm{a.s}}{=} 0$ by Eq.~\eqref{eq:deter_equiv_VGV}. Using Lem.~\ref{lem:overlap_from_resolvent} for the expression of the overlap matrix $\vect{\mu} \vect{\mu}^{\top}$, the first term $\mat{\tilde{G}}$ in the sum of Eq.~\eqref{eq:VGV_tildeG} does not contribute to the limit and one has:
\begin{align}
\label{eq:mumuT_lim_Resolvent}
    \vect{\mu} \vect{\mu}^{\top} \xrightarrow[N \to \infty]{\mathrm{a.s}} \mat{D}_{\vect{g}(1)} \matOme \lim_{z \to 1} \left[ (z-1)  ( \mat{I}_K - \mat{D}_{\vect{g}(z)} \matOme)^{-1} \right] \mat{D}_{\vect{g}(1)} \, ,
\end{align}
where we have also used the fact that the only term having a pole at one in the large $N$ limit is the matrix $ ( \mat{I}_K - \mat{\Tilde{G}} \matOme)^{-1}$ and we have used again the limit Eq.~\eqref{eq:deter_equiv_VGV} to simplify the expression.  

Since the limit $z \to 1$ does not depend on the direction, we can set $z = 1 + \epsilon$ and let $\epsilon \to 0^+$ without loss of generality. By Lem.~\ref{lem:g1_SNR_g_1} we have $\vect{g}(1 + \epsilon) \succ \vect{0}_K$ and thus $ \mat{D}_{\vect{g}(1 + \epsilon)} \matOme$ is similar to the symmetric matrix $  \mat{D}^{1/2}_{\vect{g}(1 + \epsilon)} \matOme  \mat{D}_{\vect{g}(1 + \epsilon)}^{1/2}$ and hence is diagonalizable with real eigenvalues. If we denote by $ \phi_i(\lambda)$,  $ \vect{w}_i^{(r)}(\lambda), \vect{w}_i^{(l)} (\lambda)$, the eigenvalues and right and left eigenvector of   $ \mat{D}_{\vect{g}(\lambda)} \matOme$,  we have
 \begin{align}
 \label{eq:decomposition_I-DgMat}
      ( \mat{I}_K - \mat{D}_{\vect{g}(\lambda)} \matOme)^{-1} &= \sum_{i=1}^{N} ( 1 - \phi_i(\lambda))^{-1}   \vect{w}_i^{(r)}(\lambda) (\vect{w}_i^{(l)} (\lambda))^{\top} \, .
 \end{align}
Note that  $ \phi_1(1) =1$ and is the largest eigenvalue and is simple by Lem.~\ref{lem:1_simple_eig_DgMatOme} and  $\vect{w}_i^{(r)}(1) = \vect{v}_1^{(r)}$, $ \vect{w}_i^{(l)}(1) = \vect{v}_1^{(l)}$ by Eq.~\eqref{eq:property_right_eigvect} and Eq.~\eqref{eq:property_left_eigvect}. As a consequence, in the sum of Eq.~\eqref{eq:decomposition_I-DgMat} the only contribution to the limit of Eq.~\eqref{eq:mumuT_lim_Resolvent} comes from the simple pole $(1 -\phi_1(\lambda))^{-1}$ at $\lambda =1$ and thus we have: 
\begin{align}
     \vect{\mu} \vect{\mu}^{\top} \xrightarrow[N \to \infty]{\mathrm{a.s}} \mat{D}_{\vect{g}(1)} \matOme \lim_{\epsilon \to 0^+} \frac{\epsilon}{ 1 - \phi_1(1 + \epsilon)}  \vect{v}_1^{(r)} (\vect{v}_1^{(l)})^{\top} \mat{D}_{\vect{g}(1)} \, .
\end{align}
If we introduce $\phi_1'(1)$ as the derivative of $\phi_1(.)$, which is well defined by perturbation theory for symmetric matrices (see App.~\ref{sec:proof_value_of_constant_of_proportionality}), we have:
\begin{align}
     \vect{\mu} \vect{\mu}^{\top} \xrightarrow[N \to \infty]{\mathrm{a.s}} \frac{-1}{\phi_1'(1)} \mat{D}_{\vect{g}(1)} \matOme  \vect{v}_1^{(r)} (\vect{v}_1^{(l)})^{\top} \mat{D}_{\vect{g}(1)} \, ,
\end{align}
and using again Eq.~\eqref{eq:property_right_eigvect} and Eq.~\eqref{eq:property_left_eigvect} allow us to conclude the proof. 

\section{Proof of Prop.~\ref{prop:value_of_norm_constant_overlap} (Value of the Derivative of the Top Eigenvalue)}
\label{sec:proof_value_of_constant_of_proportionality}
We recall the standard result concerning the first-order perturbation theory for eigenvalues of symmetric matrices. 
\begin{lemma}[Hadamard First Variation Formula]
\label{lem:first_order_perturbation}
    Let $\mat{A}(t)$ be symmetric and smooth for its parameter $t$. For $t_0$, assume that $\lambda_i(t_0) \equiv \lambda_i \big( \mat{A}(t_0) \big) $ is simple with associated eigenvector $\vect{v}_i(t_0)$, then $\lambda_i'(t_0)$ is well-defined and given by $\lambda_i'(t_0) = \langle \vect{v}_i(t_0), \mat{A}'(t_0) \vect{v}_i(t_0) \rangle $.
\end{lemma}
\begin{proof} See for example Chap.~1 of \cite{taoBook}. This is simply obtained by combining the differentiation of the eigenvalue equation $\mat{A}(t)  \vect{v}_i(t) = \lambda_i(t)   \vect{v}_i(t)$ and the differentiation of the normalization of the top eigenvector $\langle \vect{v}_i(t), \vect{v}_i(t) \rangle =1$. 
\end{proof}
\jump In particular, since $\phi_1(1 + \epsilon)$ is the top eigenvalue of the symmetric matrix  $  \mat{D}^{1/2}_{\vect{g}(1 + \epsilon)} \matOme  \mat{D}_{\vect{g}(1 + \epsilon)}^{1/2}$ and $\phi(1)$ is simple, we get after expressing everything in term of the right and left eigenvector $\vect{v}^{(r)}_1, \vect{v}^{(l)}_1 $, the simple formula:
\begin{align}
    \lambda_1'(1) &= \langle \vect{v}^{(l)}_1 , \left( \mat{D}_{\vect{g'}(1)} \matOme \right)^{\top} \vect{v}^{(r)}_1 \rangle \, .
\end{align}
Furthermore, we see from App.~\ref{sec:proof_Linear_System_QVE_real_line}, that $\vect{g'}(1)$ is by definition given by the solution of
\begin{align}
    \vect{1}_K = (\nabla_{\vect{g}} f)(\vect{g}(1)) \,  \vect{g'}(1) \, ,
\end{align}
and since $ (\nabla_{\vect{g}} f)(\vect{g}(1))$ is invertible for $\lambda_1(\matOme) >1$, we get using its expression \eqref{eq:def_jacobian}:
\begin{align}
    \vect{g}'(1) = (\mat{D}^{-2}_{\vect{g}(1)} - \matGam)^{-1} \vect{1}_K \, .
\end{align}
Using the explicit expressions \eqref{eq:def_left_eigvect} \eqref{eq:def_right_eigvect} for the left and right eigenvector gives us:  
\begin{align}
    (-\lambda_1'(1)) &= \langle \vect{1}_K -\vect{g}(1) , \mat{P}  (\vect{1}_K -\vect{g}(1)) \rangle  \, ,
\end{align} 
with 
\begin{align}
    \mat{P} &= \mat{D}_{\vect{\rho}^{\odot 1/2}} \matOme  \mat{D}_{-\vect{g}'(1)} \matOme  \mat{D}_{\vect{\rho}^{\odot 1/2}}
\end{align}
which gives the desired result using the definition~\eqref{eq:def_matGam} of $\matGam$ and the commutativity of diagonal matrices.

\newpage

\section{Properties of the Signal-to-Noise Ratio \texorpdfstring{$\lambda_1(\matOme)$}{}. }
\label{sec:prop_SNR}
\emph{In this Appendix, we describe simple properties of the Signal-to-Noise Ratio. We recall the assumptions $\rho_k>0$ and $s_{kl} >0$ for any $1 \leq k,l \leq K$ which implies that $\lambda_1(\matOme)$ is simple with top eigenvector $\vect{v}_1 \succ \vect{0}$ by Perron-Frobenius theorem.}
\jump
As described in Thm.~\ref{thm:main_thm}, the top eigenvalue $\lambda_1(\matOme)$ plays the role of the signal-to-noise ratio (SNR) for the inhomogeneous spiked model. We first prove that as we decrease the noise level in any block, which translates into increasing any $s_{kl}$ (recall $Y_{ij} = (x_i x_j)/\sqrt{N} + H_{ij} \Sigma_{ij}^{-1/2}$ with $\Sigma_{ij}= s_{kl}$ for any $(i,j) \in B_k \times B_l$), we improve the SNR, as one should expect.
\begin{proposition}
For any $k,l \in \{1,\dots,K \}^2$, $\lambda_1(\matOme)$ is an increasing function of  $s_{kl}$. 
\end{proposition}
\begin{proof}
The matrix $\matOme$ is continuous and differentiable for any $s_{kl}$, with derivative given by $\frac{\mathrm{d}}{\mathrm{d}(s_{kl})} \matOme = \sqrt{ \rho_k \rho_l} (\vect{e}_k \vect{e}_l^{\top} + \vect{e}_l \vect{e}_k^{\top}) $. Since the top eigenvalue is simple, we have by Hadamard's first variation formula (see Lem.~\ref{lem:first_order_perturbation}) that $\lambda_1(\matOme)$ is differentiable for $s_{kl}$ with derivative given by
\begin{align}
    \frac{\mathrm{d}}{\mathrm{d}(s_{kl})} \lambda_1( \matOme) = \langle \vect{v}_1,  \sqrt{ \rho_k \rho_l} (\vect{e}_k \vect{e}_l^{\top} + \vect{e}_l \vect{e}_k^{\top}) \vect{v}_1 \rangle = 2 \sqrt{\rho_k \rho_l} v_{1k} v_{1l} >0, 
\end{align}
where in the last equality, we have used that  $\vect{v}_1 = (v_{11},\dots, v_{1K}) \succ \vect{0}$ by the Perron-Frobenius theorem. 
\end{proof}
\jump 
Recall the partition of the signal into its $K$ parts,  $\vect{x} = \ \vect{x}_1  \oplus \dots \oplus \vect{x}_K$ where for any $k \in \{1,\dots,K\}$, the  \emph{$k$-th species} is given as $\vect{x}_k =(\vect{x})_{i \in B_k}$ and 
for two vectors $\vect{x}, \vect{y} \in (\mathbb{R}^{N}, \mathbb{R}^M)$, $\vect{x} \oplus \vect{y} = (x_1,\dots,x_N,y_1,\dots,y_M) \in \mathbb{R}^{N+M}$. 
Next, let's consider that \emph{one is not interested in retrieving the entire signal $\vect{x}$ but only a reduced version of it containing some $K' <K$ species}, that is explicitly one is interested in retrieving $\vect{x}^{(\mathrm{rdc})} = \bigoplus_{i \in I} \vect{x}_i$  for some  $I \subset \{1,\dots,K\}$ such that $|I|=K'$. A typical example corresponds to the case $\vect{x}^{(\mathrm{rdc})} = \vect{x}_1$ where one is only interested in retrieving one species out of the $K$ species. One may then ask if it is best
\begin{itemize}
    \item to consider the transformed matrix $\rdmmat{\Tilde{Y}}$ on the \emph{entire data matrix} $\rdmmat{Y}$ of size $(N \times N)$, and then keep only the overlap with the reduced signal,
    \item or to consider the transformed matrix $ \rdmmat{\tilde{Y}}^{(\mathrm{rdc})} $ obtained from the reduced data matrix where one removes any component $Y_{ij}$ such that $(i,j) \notin I^2$. The previous example with $K'=1$ would correspond to looking at the \emph{homogeneous spiked model} made of the top $( |B_1| \times | B_1|)$ left corner of the data matrix $\rdmmat{Y}$, instead of the entire matrix  $\rdmmat{Y}$.
\end{itemize}
It is natural to expect the second option to be suboptimal to the first one as it loses the information on cross terms $ Y_{ij}$ where $i \in I$ and $j \notin I$ which might help retrieve the reduced signal $\vect{x}^{(\mathrm{rdc})}$. Our next result indicates that this is indeed the case, at the level of the SNR.
\begin{proposition}
\label{prop:SNR_reduced}
If $\mat{\Omega}^{(\mathrm{rdc})}_{K'}$ is the parameter matrix for the reduced  data matrix $\rdmmat{Y}^{(\mathrm{rdc})}$, then we have:
\begin{align}
    \lambda_1( \mat{\Omega}^{(\mathrm{rdc})}_{K'}) < \lambda_1( \matOme)  \, . 
\end{align}
\end{proposition}

\begin{proof}
    By symmetry and induction, it is enough to prove $\lambda_1( \mat{\Omega}^{(\mathrm{rdc})}_{K-1}) < \lambda_1( \matOme) $, where $\mat{\Omega}^{(\mathrm{rdc})}_{K-1}$ corresponds to the parameter matrix for the reduced data matrix without the $|B_K|$ last rows and columns of the data matrix $\rdmmat{Y}$. For $N' = N- |B_K|$, we thus have 
    \begin{align}
        (Y_{ij}^{(\mathrm{rdc})})_{1 \leq i,j \leq N'} = \Big( \frac{x_i x_j}{\sqrt{N}} + H_{ij} \Sigma_{ij}^{-1/2} \Big)_{1 \leq i,j \leq N'} \, , 
    \end{align}
    In particular, note that in the large $N$ limit, the non-zero eigenvalue of the rank-one matrix has been reduced by a factor $N'/N <1$ since $\vect{x}^{(\mathrm{rdc})} \in \mathbb{R}^{N'}$.  If we normalize the matrix 
    \begin{align}
       \rdmmat{Y}^{(0)} = \sqrt{\frac{N}{N'}}  \rdmmat{Y}^{(\mathrm{rdc})} \, , 
    \end{align}
    we get back to our original setting with $K-1$ species. If we denote by $\alpha := \lim_{N\to \infty} N'(N)/N \in (0,1)$,  the parameters $(\rho_{k}^{(\mathrm{rdc})}, s_{kl}^{(\mathrm{rdc})})_{1 \leq k,l \leq K-1} $ for this reduced model are related to the ones $(\rho_{k}, s_{kl})_{1 \leq k,l \leq K}$  of the entire model by: 
    \begin{itemize}
        \item $s_{k,l}^{(\mathrm{rdc})} = \frac{ s_{k,l}}{\alpha}$ for any $1 \leq k,l \leq K-1$;
        \item $\rho_{k}^{(\mathrm{rdc})}= \alpha \rho_k$ for any $1 \leq k \leq K-1$.
    \end{itemize}
    Since $(\mat{\Omega}^{(\mathrm{rdc})}_{K-1})_{kl} = \sqrt{\rho_{k}^{(\mathrm{rdc})} \rho_{l}^{(\mathrm{rdc})} } s_{k,l}^{(\mathrm{rdc})} = \sqrt{\rho_k \rho_l} s_{kl}$,  the parameter matrix $\mat{\Omega}^{(\mathrm{rdc})}_{K-1}$ is simply the top left $\big( (K-1) \times (K-1) \big)$ minor of  $\mat{\Omega}_{K}$. We can then use the same argument as for the proof of Cauchy interlacing theorem: if we denote by $\vect{v}_1^{(\mathrm{rdc})} \succ \vect{0}$ the top eigenvector of $\mat{\Omega}^{(\mathrm{rdc})}_{K-1}$, and $\vect{v}_1$ the one of $\matOme$, we have:
    \begin{itemize}
        \item on the one hand:
        \begin{align}
        \lambda_1(\mat{\Omega}^{(\mathrm{rdc})}_{K-1}) = \langle \vect{v}_1^{(\mathrm{rdc})} , \mat{\Omega}^{(\mathrm{rdc})}_{K-1}  \vect{v}_1^{(\mathrm{rdc})} \rangle =  \langle (\vect{v}_1^{(\mathrm{rdc})} \oplus 0), \mat{\Omega}_{K}  (\vect{v}_1^{(\mathrm{rdc})} \oplus 0) \rangle \, , 
        \end{align} 
        where we recall $ (\vect{v}_1^{(\mathrm{rdc})} \oplus 0) = (v^{^{(\mathrm{rdc})}}_{1,1},\dots, v^{^{(\mathrm{rdc})}}_{1,K-1},0)$.
        \item on the other hand:
        \begin{align}
        \lambda_1(\mat{\Omega}_{K}) = \mathrm{max}_{\vect{v}, \| \vect{v} \| =1} \langle \vect{v} , \mat{\Omega}_{K}  \vect{v} \rangle =  \langle \vect{v}_1 , \mat{\Omega}_{K}  \vect{v}_1 \rangle \, .
        \end{align} 
    \end{itemize}
Since $\vect{v}_1 \succ \vect{0}$, we must have $\vect{v}_1 \neq  \vect{v}_1^{(\mathrm{rdc})}$ leading to the desired result since $ \lambda_1(\mat{\Omega}_{K})$ is simple.
\end{proof}
\jump
Note that Prop.~\ref{prop:SNR_reduced} indicates that reducing the data matrix might lead to losing any correlation between the top eigenvector and the signal since for a model with parameters such that $\lambda_1( \matOme) >1$, the overlap between the top eigenvector and \emph{each} species $\vect{x}_k$ is positive, and in particular one has also a positive overlap between the (entire) top eigenvector and the reduced signal vector  $\vect{x}^{(\mathrm{rdc})}$. By reducing the data matrix, one may end up with $ \lambda_1( \mat{\Omega}^{(\mathrm{rdc})}_{K'}) < 1$ such that there is no positive overlap between the associated top eigenvector $\vect{v}_1^{(\mathrm{rdc})}$ and the reduced signal $\vect{x}^{(\mathrm{rdc})}$.


\end{document}